\documentclass[11pt,a4paper]{article}

 \usepackage[utf8]{inputenc} %
 \usepackage[T1]{fontenc}    %
\usepackage[margin=1in]{geometry}
\usepackage{hyperref}       %
 \usepackage{booktabs}       %
 \usepackage{amsfonts}       %
 \usepackage{nicefrac}       %
 \usepackage{microtype}      %
 \usepackage{xcolor}         %
\usepackage{enumitem}

\usepackage{subcaption}
\usepackage{stfloats}

\usepackage{amsthm, amssymb}
\usepackage{amsmath}
\usepackage{graphicx}
\usepackage{xspace}
\usepackage{enumitem}
\usepackage{tabularx}
\usepackage{xcolor}
\usepackage{graphicx} 
\usepackage{tikz}
\usepackage{bm}
\usepackage{wrapfig}
\usepackage{thm-restate}

\usepackage{booktabs}
\usepackage[table]{xcolor}

\theoremstyle{plain}
\newtheorem{theorem}{Theorem}[section]
\newtheorem{lemma}[theorem]{Lemma}

\newtheorem{observation}[theorem]{Observation}

\newtheorem{corollary}[theorem]{Corollary}
\newtheorem{claim}[theorem]{Claim}
\theoremstyle{definition}
\newtheorem{definition}[theorem]{Definition}

\newcommand{\R}{\mathbb{R}}

\newcommand{\calP}{{\mathcal P}}

\newcommand{\eps}{\varepsilon}

\newcommand{\LP}{\mathsf{LP}}
\newcommand{\DP}{\mathsf{DP}}
\newcommand{\CP}{\mathsf{CP}}

\newcommand{\opt}{\mathsf{opt}}
\newcommand{\alg}{\mathsf{alg}}

\newcommand{\hy}{\hat y}
\newcommand{\hx}{\hat x}
\newcommand{\hp}{\hat p}

\newcommand{\bone}{\textsf{1}}

\newcommand{\pf}{\mathrm{PF}}

\newcommand{\PM}{\mathsf{PM}} 
\usepackage[
 backend=biber, style=alphabetic, citetracker, hyperref=true,
 maxcitenames=2, %
 sortcites,      %
 sorting=nyt,
 maxbibnames=12, %
 giveninits,     %
 date=year,      %
 isbn=false,     %
 url=false,      %
 doi=false,      %
 eprint=false,   %
]{biblatex}

\newbibmacro{string+doiurlisbn}[1]{
  \iffieldundef{doi}{
    \iffieldundef{url}{
      #1
    }{
      \href{\thefield{url}}{#1}
    }
  }{
    \href{http://dx.doi.org/\thefield{doi}}{#1}
  }
}

\DeclareFieldFormat
[article,inbook,incollection,inproceedings,patent,thesis,unpublished]
  {title}{\usebibmacro{string+doiurlisbn}{#1}}

\usepackage{hyperref}
\hypersetup{colorlinks,citecolor=green!60!black,linkcolor=blue!70!black,urlcolor=blue!70!black,breaklinks=true}
\usepackage[capitalize, noabbrev]{cleveref}
\usepackage{xurl}

\usepackage[colorinlistoftodos,prependcaption,textsize=scriptsize]{todonotes}

\title{Learning-Augmented Online Scheduling\\ with Parsimonious Preemption}

\author{Mugen Blue\thanks{University of California, Santa Cruz, US. \texttt{(msblue@ucsc.edu, sim9@ucsc.edu)}. Supported in part by NSF awards CCF-2535599, 2537126, and 2423106, and by ONR grant N00014-22-1-2701.} \and Sungjin Im\footnotemark[1] \and Alexander Lindermayr\thanks{Institut für Mathematik, Technische Universität Berlin, Germany. \texttt{alexander.lindermayr@tu-berlin.de}.}}
\date{}

\begin{document}

\maketitle

\begin{abstract}
Learning-augmented algorithms have emerged as a powerful 
paradigm to surpass traditional worst-case lower bounds 
by integrating potentially noisy predictions. While this 
framework has seen success in online scheduling, existing work 
primarily optimizes job latency while relying on 
frequent, ``blind'' preemptions. This ignores the fundamental 
trade-off between algorithmic performance and preemption complexity.
We provide the first systematic study of learning-augmented 
scheduling that curbs preemption while optimizing latency. We 
establish that the gap between theoretical latency bounds and 
preemption overhead can be bridged with solid analytical foundations.
Our results include $O(1)$-competitive algorithms for single and 
unrelated parallel machines with only $O(1)$ preemptions per job under 
accurate predictions, with overhead scaling logarithmically with 
the prediction error. By providing the first bounded-preemption guarantees for unrelated and malleable machines, we extend the theoretical reach of the learning-augmented framework to more constrained and realistic settings. Finally, our algorithms are validated through experiments. 
\end{abstract}

\section{Introduction}

The paradigm of \emph{learning-augmented algorithms} has recently emerged as a powerful framework for overcoming the limitations of worst-case analysis \cite{mitzenmacher2022algorithms}. By incorporating potentially noisy machine-learning predictions into algorithm design, this framework enables algorithms that are both \emph{consistent}---near-optimal when predictions are accurate---and \emph{robust}---with worst-case guarantees even under adversarial predictions. This approach is especially effective for online algorithms, where predictions can mitigate future uncertainty that is unavoidable in traditional worst-case models.

Online scheduling is a prominent application of this paradigm. Traditionally, algorithms often suffer from limited information about key parameters; a central example is the \emph{non-clairvoyance}, where job sizes are unknown until completion \cite{MPT94, edmonds1999scheduling}. A growing literature shows that online algorithms can substantially benefit from job-size predictions when minimizing latency \cite{purohit2018improving, mitzenmacher2019scheduling, ImKQP23, azar2021flow, ZLZ22}.

However, in practical systems such as large-scale data centers, \emph{preemption} and \emph{migration} can incur substantial overheads, including context-switching costs and cache invalidation \cite{mogul1991effect, tsafrir2007context, verma2015large}. Although preemption-aware scheduling is well studied in systems research, from mechanisms such as preemption thresholds \cite{saksena2000scalable} and deferred-preemption heuristics \cite{davis2012optimal} to production-scale data-center resource management \cite{verma2015large, chen2017preemptive}, these approaches are largely qualitative and offer few provable guarantees on job latency, particularly under prediction error.

We argue that the absence of an analytical framework for this trade-off remains a significant gap in the learning-augmented literature. We address this gap by using predictions to design the first algorithms that minimize total completion time while using preemption parsimoniously. To our knowledge, this is the first systematic study to analytically ground the trade-off between job latency and preemption complexity under job-size predictions. Our results move beyond the ``blind'' preemption required by classical non-clairvoyant models and provide the first bounded-preemption guarantees for unrelated and malleable machine settings.

\subsection{Summary of our results}

\noindent
\textbf{Problem setup:} We assume all $n$ jobs are available at time $t=0$, because %
i) simultaneous arrival serves as a fundamental baseline in scheduling theory; and  ii) while extensions to arbitrary arrival times are possible, they can obscure the core analytical ideas. We therefore provide a sketch of the extensions in the appendix to maintain a transparent presentation of our main results. %

For each job $j$, we know its predicted requirement $\hat{p}_j$. In practice, these predictions $\hat{p}_j$ are typically derived by profiling the resource consumption of recurring jobs or by applying statistical models to historical execution traces from the cluster \cite{faisal2024will,park20183sigma}. A job~$j$ completes once it has been processed by its requirement $p_j$. 
However, the exact processing requirements remain unknown until completion; this setting is referred to as \emph{non-clairvoyant scheduling} (when predictions are unavailable), in contrast to \emph{clairvoyant scheduling}, where job sizes are known (upon  arrival). 

We focus on minimizing the total completion time objective, $\sum_j C_j$, where $C_j$ is the completion time of job $j$. Since we evaluate the online algorithm relative to an optimum multiplicatively, our goal is equivalently stated as minimizing the average completion time. An algorithm is %
$c$-\emph{competitive} if its total completion time is at most $c$ times that of the offline optimum for all %
inputs.

 We first present our findings under the assumption that all predictions are underestimates (i.e., $\hat{p}_j \leq p_j$ for all jobs $j$). Subsequently, we introduce a robust framework to handle overestimated job requirements. We define $D(p, \hat{p}) := \sum_j (\log_2 (p_j / \hat{p}_j) + 1)$ as the ideal benchmark for the number of preemptions; we denote this simply as $D$ when the context is clear. This benchmark is motivated by the fact that any non-clairvoyant algorithm performing $o(\log p_j)$ preemptions, even for a single job $j$, necessarily has a competitive ratio of $\Omega(n)$ (see Theorem~\ref{thm:lower-bound-nonclairvoyant} in the appendix). Consequently, if $\hat{p}_j = 1$ for all $j$, any algorithm must perform $\Omega(D)$ preemptions to remain competitive.

Our results are summarized below. This is the \textit{first} work to demonstrate how predictions can be utilized to minimize preemption. We provide constant-competitive algorithms while bounding the number of preemptions per job to within a constant factor of the benchmark $D$. 

\begin{itemize}
    \item \textbf{Single machine:}  
    In the single-machine setting, which serves as a warm-up for our study, only one job can be processed at a time. For any $\varepsilon > 0$, we provide an algorithm that is $(2 + \varepsilon)$-competitive while incurring at most $O(\frac{1}{\varepsilon} D)$ preemptions (Theorem~\ref{thm:single}). This guarantee extends to identical parallel machines (Theorem~\ref{thm:identical-parallel}).

    \item \textbf{Unrelated machines:} This is our main result. In this  heterogeneous parallel machine environment, each job~$j$ is processed at a rate of  $\lambda_{ij}$ on machine $i \in M$. A machine can process at most one job at a time. For this problem, we give a $14.48(1+\varepsilon)$-competitive algorithm that performs at most $O(\frac{1}{\varepsilon^3} D)$ preemptions (Theorems~\ref{thm:u-preemption-bound} and \ref{thm:u-cr}).
    
    \item \textbf{Malleable job scheduling:} In this setting, jobs can be assigned to multiple machines (processors) simultaneously to expedite processing subject to eligibility constraints. Each job $j$ is associated with a concave speed-up function that determines the processing rate based on the number of machines assigned. 
    For any $\varepsilon > 0$, we provide an algorithm that is $27 \frac{e}{e-1}(1+\varepsilon)$-competitive, making at most $O(\frac{1}{\varepsilon^3} D)$ preemptions (Theorem~\ref{thm:m-u}); see Theorem~\ref{thm:m-i} for a slightly better bound on identical parallel machines. 
\end{itemize}

We remark that our results extend readily to weighted jobs by considering the weighted objective $\sum_j w_j C_j$ and the weighted preemption benchmark $D(p, \hat{p}) := \sum_j w_j (\log_2 (p_j / \hat{p}_j) + 1)$. For clarity of presentation, we focus on the unweighted case throughout the paper.

We also validate our algorithms through experiments, including heuristic adaptations of the theoretical versions developed in this paper. The results demonstrate that our algorithms effectively optimize total completion time while utilizing preemptions parsimoniously

\paragraph{Handling overestimation.} Interestingly, overestimates are not symmetric to underestimates. To minimize total completion time, one must intuitively complete  many jobs as quickly as possible~\cite{smith1956various}. %
Without knowledge about which jobs are short, we must hedge against short jobs being delayed by long ones via exponentially increasing our guesses; if a job does not complete after $(1+\delta)^k$ units of processing, we switch to other jobs, assuming its requirement is at least $(1+\delta)^{k+1}$, for some constant $\delta > 0$. Underestimations allow us to skip many initial guesses, but overestimations are less helpful: %
if all jobs are overestimated to the same requirement, the predictions provide no useful prioritization.

In fact, if we treat overestimated jobs identically to underestimated ones, the competitive ratio can degrade to $O(\max_j \hat{p}_j / p_j)$. A more robust approach is to \emph{intentionally underestimate} job requirements by scaling down the initially provided estimates. This ensures that only a small number of jobs, $g$, remain overpredicted. We execute the algorithm for underestimated jobs until the number of unfinished jobs drops to $g/\varepsilon$. At this transition point, we reset the predictions for all remaining jobs to ensure they are henceforth treated as underestimates. %
This strategy maintains the competitive ratio while incurring only $O(\frac{g}{\varepsilon^2} \log_2 p_{\max})$ additional preemptions; %
see Theorem~\ref{thm:single-weakly-enforce} for details. 

As long as $g$ is sublinear in $n$ and job requirements are polynomial in $n$, a common and realistic scenario, the impact of overestimated jobs on the total number of preemptions is asymptotically negligible.
Furthermore, a single global parameter $g$, which represents an upper bound on the number of overestimations, is arguably easier to predict accurately than jobs' individual sizes. This strategy successfully extends to the other settings considered in this work (Section~\ref{sec:u-over}).

\subsection{Our approaches and contributions}\label{sec:intro-techniques}

To understand the challenge of minimizing preemptions, note that we cannot use the popular Round-Robin (RR) algorithm. 
RR cycles through active jobs and processes each for one unit, resulting in prohibitively many preemptions. 
Despite this drawback, RR has been widely used in the learning-augmented %
literature because it is memoryless: 
its execution depends only on the set of currently active jobs~\cite{purohit2018improving,LindermayrM25,ImKQP23}. Thus, if each instantaneous schedule 
approximates RR, we obtain a comparable performance guarantee; this is the basis of our approach.

To avoid excessive preemptions, we consider a less preemptive algorithm called Multi-Level Feedback (MLF) \cite{MPT94}. 
MLF maintains multiple FIFO queues to process jobs; for some constant $\delta > 0$, the $k$-th queue contains 
jobs that have been processed for at least $(1+\delta)^{k}$ and less than $(1+\delta)^{k+1}$ units of time. 
A non-clairvoyant algorithm, given no predictions, initially 
places each job in the queue corresponding to the smallest $k \geq 0$.

In our variant, which we call \textbf{Predicted Multi-Level Feedback (PMLF)}, a new job starts in the queue 
corresponding to its predicted requirement $\hat{p}_j$, rather than in the first queue. If all predictions are 
underestimates ($\hat{p}_j \leq p_j$), we show that PMLF achieves an $O(1)$-competitive ratio 
while incurring $O(D) = O( \sum_j (1 +  \log_{1+\delta} (p_j / \hat{p}_j) ) )$ preemptions. Hence, when predictions are accurate, we perform only $O(1)$ preemptions per job on average.

Our main result is the \textbf{Simulated Non-preemptive Adaptive Prediction (SNAP)} algorithm for unrelated 
machines. SNAP is designed to handle the complexities of unrelated machines, the most general 
parallel-machine model in classical scheduling theory \cite{lenstra1990approximation,skutella2001convex,HSSW97}. Here, we cannot simply use a processing queue based 
on elapsed work, because the processing speeds $\lambda_{ij}$ are unique to each job-machine pair: a 
given machine may be highly efficient for one job and extremely slow for another. 

A prominent non-clairvoyant algorithm for unrelated machines is Proportional Fairness (PF), which identifies job processing rates that maximize their product, a task reducible to a convex program. This is not a computational bottleneck; rather, it enables efficient schedule computation, since the objective can be approximated to arbitrary precision in near-linear time \cite{criado2022fast}.
However, a PF schedule is highly fractional (it alternates between multiple feasible schedules\footnote{A feasible schedule at each time can be thought of as a matching between jobs and machines.}) and is 
therefore highly preemptive. Furthermore, PF must recompute the %
rates whenever a job arrives or completes. 

To address this, %
SNAP operates through a sequence of disjoint epochs. 
At the start of each epoch, the algorithm runs PF to compute the desired job processing rates. It then 
uses predicted job requirements to define the next checkpoint for each job: the 
smallest power of $1+\delta$ that is greater than both $\hat{p}_j$ and the current amount of 
processing the job has received. This strategy limits the total number of preemptions 
to a logarithmic factor of the multiplicative estimation error. We then choose an epoch duration so that, under PF, a $\beta$-fraction of jobs reach their next checkpoints. The preemptions 
made in the epoch are charged to the jobs that reached their checkpoints, and our algorithm remains $O(1)$-competitive 
because we follow PF for a $1-\beta$ fraction of jobs. Finally, we convert this virtual %
preemptive 
schedule into a non-preemptive schedule using an offline rounding procedure and run it. %

For malleable jobs, however, we use a different approach. The reason is that the approximate 
super-additivity property is not known to hold for this problem, unlike in the unrelated-machines setting 
(see Lemma~\ref{lem:decomposition} for details). Therefore, we first analyze a more robust version 
of the %
PF algorithm. It is known that PF is $O(1)$-competitive for a general class of scheduling problems, called Polytope Scheduling (see Section~\ref{sec:framework-psp} for the definition) \cite{ImKM18,JLM25}, which encompasses the malleable-jobs setting. 
To make PF less preemptive, we consider a version 
that does not change its rates until a constant fraction of jobs reach their next checkpoint. 
Here, unlike in SNAP, where we use the known performance guarantee of PF as a 
black box, we directly modify the analysis of PF from \cite{ImKM18,JLM25}. This approach incurs at most 
$O(\frac{1}{\varepsilon^3}D)$ processing rate changes in total, and by simulating this algorithm as in SNAP with 
predictions, we obtain the desired result.

\subsection{Further related work}

We focus on the most closely related work here and defer a broader discussion to Appendix~\ref{app:related}. 

\textbf{Single and identical machines.} In the clairvoyant setting, Shortest Job First (SJF), which processes jobs in non-decreasing order of processing time, is optimal for total completion time~\cite{smith1956various}. In the non-clairvoyant setting, RR is best-possible $2$-competitive \cite{MPT94}, but incurs excessively many preemptions. Multilevel (Adaptive) Feedback (MLF) reduces preemptions while remaining $O(1)$-competitive, and can be parameterized to be $(2+\varepsilon)$-competitive for any $\varepsilon > 0$ \cite{MPT94}. These guarantees for SJF and RR %
extend to parallel identical machines \cite{MPT94, BBEM12}.

\textbf{Unrelated machines.}  In the clairvoyant online setting, utilizing immediate-dispatch and non-migratory algorithms while assuming each machine runs SPT, %
assigning an arriving job to the machine with the least marginal increase in the objective is $O(1)$-competitive \cite{AnandGK12,GuptaMUX20,LindermayrM25}. For non-clairvoyant algorithms, SelfishMigrate and Proportional Fairness (PF) were analyzed with competitive ratios of $32$ and $3.62$, respectively \cite{ImKMP14, JLM25}. Clairvoyant algorithms were also explored in \cite{HSSW97, ChakrabartiPSSSW96, LindermayrMR23}.

\textbf{Online scheduling with predictions.} 
Various objectives have been considered in the learning-augmented setting, see \cite{alps} for an overview. Examples include total completion time \cite{purohit2018improving, ImKQP23, LindermayrM25, DinitzILMV22,LassotaLMS23}, total flow time \cite{azar2021flow, mitzenmacher2019scheduling}, makespan minimization \cite{lattanzi2020online, li2021online, BalkanskiOSW25}, and energy consumption \cite{bamas2020learning, AntoniadisGS22}.

\section{Single machine}

We start with our first result, which is for scheduling
on a single machine. 
We first present our algorithm Predicted Multi-Level Feedback (PMLF) %
for a hyperparameter $\delta > 0$: 
\begin{itemize}[nosep]
	\item Maintain a set of FIFO queues $\{Q_i\}_{i \geq 0}$. Initially, place a job $j$ into $Q_k$ where $\lfloor \log_{1+\delta} \hat p_j \rfloor$.
	\item At any time $t$, work on the front job of the non-empty queue of the lowest index.
	\item If $j \in Q_i$ was processed by in total by $(1+\delta)^{i+1}$, move $j$ from $Q_i$ to the end of $Q_{i+1}$.
\end{itemize}
PMLF becomes MLF when $\hat p_j = 1$ for each job $j$ (under the standard assumption that $p_j \geq 1$).
For now, we assume that $\hp_j \leq p_j$ for each job~$j$; we will soon discuss how to relax this assumption.

\begin{restatable}{theorem}{thmSingle}
    \label{thm:single}
    For any constant $\delta > 0$,
PMLF is $(2+2\delta)$-competitive for minimizing the total completion time on a single machine and %
performs at most $O(\frac{1}{\delta}\log_2(p_j / \hat p_j))$ preemptions for each job $j$ if there are only underpredictions.
\end{restatable}

We sketch the proof of the theorem deferring the full proof to \Cref{app:single-machine-appendix}.
The preemption bound comes from the fact that job~$j$ is preempted only when it is moved between queues,
which happens at most
$\lfloor \log_{1+\delta} p_j\rfloor - \lfloor \log_{1+\delta} \hat p_j\rfloor$ many times.
For the competitive ratio, we reuse the standard MLF delay-style analysis of~\cite{MPT94}. It is easy to observe that before a job $j$ completes, any other job $i$ can be processed for at most $\min\{ (1+\delta) p_j, p_i\}$ units of time. If $\delta$ is infinitesimally small, the analysis essentially becomes that of Round-Robin, which is known to be $2$-competitive~\cite{MPT94}. The multiplicative factor $1+\delta$ used to define queues inflates the competitive ratio by a factor of $1+\delta$.

\paragraph{Handling overestimations.}
If we allow overestimations, let $g$ be an upper bound on the number of overestimated jobs. 
By scaling down predicted sizes, we can make $g$ small at the cost of more preemptions. 
We consider a variation of PMLF, which uses $g$ and a parameter $\gamma > 0$:
Let $T$ be the time when the number of unfinished jobs is at most $g / \gamma$, and $q_j(T)$ denote the cumulative processed size of job $j$ until time $T$.
At time $T$, move each unfinished job $j$ to the end of the queue $Q_k$ 
where $(1+\delta)^k \leq q_j(T) < (1+\delta)^{k+1}$.
Then, we continue with PMLF. %
We prove in \Cref{app:single-machine-appendix}: %

\begin{restatable}{theorem}{thmSingleEnforce}\label{thm:single-weakly-enforce}
	If there are at most $g$ overestimated jobs, 
	then for every $\varepsilon > 0$, the adapted PMLF is $2(1+\varepsilon)$-competitive
	and incurs at most 
	$O ( \frac{1}{\varepsilon} \sum_{j} 
(1+\log_2 \frac{p_j}{\hp_j}) + \frac{g}{\varepsilon^2} \log_2 (\max_j p_{j}) )$  preemptions. %
\end{restatable}

\section{Unrelated machines}\label{sec:unrelated}

In this section, we describe 
SNAP for scheduling on unrelated machines.
We are given a set $M$ of unrelated machines. 
Each job $j$ is available for scheduling from time 0, and we know its predicted size $\hat{p}_j$ and processing rates $\{\lambda_{ij}\}_{i \in M}$ across all machines. %
We assume that every job $j$ has processing requirement equal to a power of $1+\delta$ for some $\delta > 0$; that is, $p_j = (1+\delta)^i$ for some integer $i \geq 0$. 
If $p_j$ is not a power of $1+\delta$, we instead treat the job as active until it has been processed for $(1+\delta)^{\lceil\log_{1+\delta} p_j\rceil}$ units of time. This increases the competitive ratio by a factor of at most $1+\delta$.
We assume that all job sizes are underestimated, i.e., $\hat{p}_j \leq p_j$ for all jobs $j$. The case of overestimation is addressed in Appendix~\ref{sec:u-over}, following the same approach as for the single-machine case.

Before describing SNAP, we introduce \emph{checkpoints}. We say that a job $j$ reaches a \emph{checkpoint} at time $t$ if its cumulative processed size $q_j(t)$ reaches a power of $1+\delta$. 
For convenience, when a job reaches a checkpoint, we say that it is \emph{exhausted} (or that an \emph{exhaustion} occurs). 
For job $j$, we begin tracking checkpoints at its predicted size $\hat p_j$: 
let $(1+\delta)^{\rho_j}$ be its first checkpoint, where 
$\rho_j := \lceil\log_{1+\delta} \hat p_j\rceil$. 
Thus, the checkpoints of $j$ are $\{(1+\delta)^{\rho'}\}_{\rho' \geq \rho_j}$, and job $j$ can reach a checkpoint at most $\lceil \log_{1+\delta} (p_j / \hat{p}_j) \rceil + 1$ times.

\subsection{Algorithm: SNAP}

Our algorithm, which we call \textbf{Simulated Non-preemptive Adaptive Prediction (SNAP)} algorithm,
proceeds in epochs $1, 2, \dots, K$. %
Epoch $k$ starts at time $e_k$ and ends at time $e_{k+1}$, at which point epoch $k+1$ begins immediately. 
The first epoch starts at time $e_1 = 0$. 
At the start of epoch $k$, let $J_k := J(e_k)$ denote the set of available jobs for processing. 
Here, $J(t) := \{ j \mid r_j \leq t < C_j \}$ is the set of jobs alive at time $t$. Let $n_k := |J_k|$. Let $\beta, \gamma \in (0, 1)$ be some constants that we optimize later.

At the beginning of each epoch, we compute a complete schedule for the duration of the epoch and follow it. Specifically, we execute the following steps:

\begin{enumerate}
    \item \textbf{Compute Proportional Fairness (PF) processing rates.} We compute the PF rates for the set $J_k$, denoted by $y^{(k)} := \text{PF}(J_k)$. Formally, $y^{(k)}$ is the solution to the following convex program with $J = J_k$:
%
%
\begin{equation*}
\begin{aligned}
\textrm{PF}(J):\quad \max \quad & \sum_{j \in J} \log(y_j) \\
\text{s.t.}\quad
& \sum_{j \in J} x_{ij} \le 1 \quad \forall i \in M, \quad \sum_{i \in M} x_{ij} \le 1,\quad y_j = \sum_{i \in M} \lambda_{ij} x_{ij} &&  j \in J, \\
& x_{ij} \ge 0 \quad \forall i \in M,\ j \in J.
\end{aligned}
\end{equation*}
This can be solved approximately in near linear time \cite{criado2022fast}.

%
    \item \textbf{Compute remaining requirement to the next checkpoint.} 
    For each job $j \in J_k$, we define the remaining processing requirement to reach its next checkpoint as $u_{j,k} := (1 + \delta)^{h+1} - q_j(e_k)$, where $h \in \mathbb N$ %
    such that $(1 + \delta)^h \leq \max\{\hat p_j, q_j(e_k) \} < (1 + \delta)^{h+1}$.

    \item \textbf{Determine the simulation duration.} Let $l_k$ denote the length of the shortest time interval such that, if jobs are processed at rates $y^{(k)}$ during the time interval, at least $\lceil \beta n_k \rceil$ jobs from $J_k$ reach their next checkpoint. Formally, let the jobs in $J_k$ be sorted in non-decreasing order of $u_{j,k} / y^{(k)}_j$. If $j^*$ is the $\lceil \beta n_k \rceil$-th job in this ordering, breaking ties arbitrarily, then the duration is defined as $l_k = u_{j^*, k} / y^{(k)}_{j^*}$. The targeted processing requirement for each job $j \in J_k$ in the current epoch is then $v_{j,k} = \min\{u_{j,k}, l_k y^{(k)}_j\}$.

    \item \textbf{Convert to a non-preemptive schedule.} 
    Since the fractional assignments $x_{ij}$ from Step~1 constitute a valid 
    preemptive schedule for the required processing requirement 
    $v_{j,k}$, we can construct %
    a non-preemptive (and non-migratory) schedule
    that begins 
    at time $e_k$ and completes by time $e_k + 4 \cdot l_k$ where each each job $j \in J_k$ is processed by $v_{j,k}$ units~\cite{CorreaSV12}.

\end{enumerate}

We note that the algorithm runs in polynomial time, as both the conversion method~\cite{CorreaSV12} and the computation of PF rates~\cite{ImKM18,JLM25} are polynomial-time solvable.
From the description of SNAP, we can make the following two observations.

\begin{observation} \label{obs:u-exhaustion}
    In an epoch $k$, at least $\lceil \beta n_k \rceil$ jobs reach a checkpoint.
\end{observation}
\begin{observation} \label{obs:u-epoch-length}
    For each epoch $k$, we have that $e_{k+1} - e_k \leq 4 \cdot l_k$.
\end{observation}

We first argue about the number of preemptions and migrations of SNAP.

\begin{theorem}\label{thm:u-preemption-bound}
    When all job sizes are underestimated, the total number of preemptions performed by SNAP is at most 
$
O\big(  \frac{1}{\beta \delta}\sum_j  \big( \log_2 \frac{p_j}{\hat{p}_j} + 1 \big) \big).
$
\end{theorem}

\begin{proof}
    In each epoch $k$, every job $j \in J_k$ is preempted (or migrated) at most once due to the non-preemptive nature of the schedule constructed in Step 4. Thus, the total number of preemptions is bounded by $\sum_{k} n_k$. By Observation~\ref{obs:u-exhaustion}, at least $\beta n_k$ jobs reach a checkpoint during an epoch. Consequently, $\sum_{k } n_k \leq \frac{1}{\beta} \times (\text{Total number of exhaustions})$.   
    Each job $j$ can reach a checkpoint at most $\lceil \log_{1+\delta} (p_j / \hat{p}_j) \rceil + 1$ times (under the assumption of underestimated predictions). Therefore, 
    \begin{equation*}
    \textstyle
        \sum_{k } n_k \hspace{-0.5mm} \leq  \hspace{-0.5mm}\frac{1}{\beta} \sum_j \bigl( \lceil \log_{1+\delta} \frac{p_j}{\hat{p}_j} \rceil \hspace{-0.5mm}+ \hspace{-0.5mm}1 \bigr) 
        \hspace{-0.5mm} =\hspace{-0.5mm} O \bigl( \frac{1}{\beta \delta} \sum_j \bigl( 1 + \log_2 \frac{p_j}{\hat{p}_j} \bigr) \bigr) \ . \qedhere
    \end{equation*}
\end{proof}

We next aim to prove the following theorem on SNAP's competitive ratio.

\begin{theorem}\label{thm:u-cr}
	For any $\varepsilon > 0$, the SNAP algorithm with $\eta = \frac{\varepsilon}{2(1+\varepsilon)}$ and $\beta = \frac{\varepsilon^2}{4(1+\varepsilon)(2+\varepsilon)}$.
    achieves a total completion time that is at most 
    $14.48(1+\varepsilon)$ times the optimum.
\end{theorem}

We split up the proof of this theorem into multiple lemmas. From now on fix the schedule of the SNAP algorithm.
Let $K$ be the number of epochs. Let $\alg(p)$ denote SNAP's total completion time objective for a processing time vector $p = \{p_j\}_j$.
 
\begin{lemma}\label{lem:u-alg-virt-bound}
	It holds that $\alg(p) \leq 4 \sum_{k=1}^K n_k l_k$.
\end{lemma}

\begin{proof}
We have
\begin{align*}	
    \alg(p) %
    = \sum_{k=1}^K \sum_{t = e_k}^{e_{k+1} - 1} |J(t)| %
    \leq \sum_{k=1}^K \sum_{t = e_k}^{e_{k+1} - 1}|J(e_k)|  %
    =  \sum_{k=1}^K (e_{k+1} - e_k) n_k %
    \leq  4 \sum_{k=1}^K l_k n_k  \ ,
\end{align*}
where the equalities use that each alive job contributes $1$ per timestep and $n_k = |J(e_k)|$,
and the inequalities use monotonicity of $|J(t)|$ and Observation~\ref{obs:u-epoch-length}.
\end{proof}

Let $\opt(z_j \mid j \in S)$ denote the minimum total completion time to complete a set of jobs $S$, where job $j \in S$ has a size $z_j$. For the sake of analysis, consider a set of jobs where we have one job for each $j \in J_k$, but with a different size, $l_k y_j^{(k)}$. Since $y_j^{(k)}$ is PF's processing rate on $j$ when it runs on $J_k$, PF will finish them with these hypothetical sizes exactly at the same time. Using Theorem 4.1 of \cite{JLM25}, which shows that PF achieves a competitive ratio of two when it completes all jobs with the same arrival time at the same time, we have the following claim.

\begin{claim} 
    \label{claim:u-1}
For each epoch $k$,     
	we have $n_k l_k \leq 2 \cdot \opt(l_k y_j^{(k)} \mid j \in J_k)$.	
\end{claim}

Let $X_k$ be the jobs exhausted during epoch $k$, strictly before the end of the epoch. %
Let $Y_k := J_k \setminus X_k$. Note that $|Y_k| \geq (1 - \beta) n_k$ because of the definition of the epoch. Our next goal is to bound the RHS of the inequality in Claim~\ref{claim:u-1} by a linear combination of analogous quantities restricted to $X_k$ and $Y_k$, respectively. The proof is deferred to Appendix~\ref{app:u-deferred}.

\begin{restatable}{claim}{claimSnapCombination}
    \label{claim:u-2}
For each epoch $k$ and any $\eta \in (0, 1)$, we have, 
    \begin{align*} %
         \opt(l_k y_j^{(k)} \mid j \in J_k) \leq \frac{1}{\eta} \opt(l_k y_j^{(k)} \mid j \in X_k) + \frac{1}{1-\eta} \opt(l_k y_j^{(k)} \mid j \in Y_k)  \ .
    \end{align*}
\end{restatable}

The following lemma upper bounds the cost incurred by our algorithm during epoch $k$ by the cost that the optimal solution must pay for the processing performed on jobs that are not exhausted in the middle of the epoch. We defer to Appendix~\ref{app:u-deferred} the proof of the following lemma, which follows by combining Claims~\ref{claim:u-1} and \ref{claim:u-2} with %
    $|X_k| \leq \beta n_k$.

\begin{restatable}{lemma}{lemepochbound}
\label{lem:epoch-bound}
	For each epoch $k$, we have $n_k l_k \leq  \frac{2\eta}{(\eta - 2\beta) (1-\eta)} \cdot \opt(l_k y_j^{(k)} \mid j \in Y_k) $.
\end{restatable}

We then show that $\{l_k y^{(k)}\}_{j \in Y_k}$ is a decomposition of the full instance.
\begin{lemma}\label{lem:sum-up-pj}
 For every job $j$, it holds that $\sum_{k: j \in Y_k} l_k y_j^{(k)} \leq p_j$.
\end{lemma}

\begin{proof}
    Fix  a job $j$. Consider an epoch $k$ with $j \in Y_k$. Since $j$ is not exhausted until epoch $k$'s end time, it is processed by %
    $l_k y_j^{(k)}$ during the epoch. The sum of this quantity over all epochs $k$ with $j \in Y_k$ cannot exceed $j$'s size, $p_j$.    
\end{proof}

We now employ the ``decomposition'' lemma, which follows from the super-additivity of the non-preemptive optimum for unrelated machine scheduling \cite{JLM25} and the $1.81$-gap between the preemptive and non-preemptive optima for the total (weighted) completion time objective \cite{Sitters17}.

\begin{lemma}[{Lemma~4.5 in~\cite{JLM25}}]
    \label{lem:decomposition}
    Consider any set of jobs $S$ where each job $j$ has a size $p_j$. For each $j \in S$, let $p_{j, \kappa'} \geq 0$, $1 \leq \kappa' \leq \kappa$, be such that $\sum_{\kappa' = 1}^\kappa p_{j, k'} \le p_j$. Then, we have 
    $\sum_{\kappa' =1}^\kappa \opt(p_{j,\kappa'} \mid j \in S) \leq 1.81 \cdot \opt(p_j \mid j \in S)$.
\end{lemma}

Lemmas~\ref{lem:sum-up-pj} and ~\ref{lem:decomposition} immediately yield the following. 

\begin{corollary}
    \label{cor:u-final}
    It holds that $\sum_{k=1}^K \opt(l_k y_j^{(k)} \, | \, j \in Y_k) \leq 1.81 \cdot \opt(p)$.
\end{corollary}

We are now ready to %
prove Theorem~\ref{thm:u-cr}. %
Using \Cref{lem:u-alg-virt-bound}, 
     \Cref{lem:epoch-bound}, 
    and \Cref{cor:u-final}, we obtain
    \begin{align*}
    \alg(p) 
    \leq 4 \textstyle \sum_{k=1}^K n_k l_k 
    &\leq 4 \big( \tfrac{2\eta}{(\eta - 2\beta) (1-\eta)} \big) \textstyle \sum_{k=1}^K \opt(l_k y_j^{(k)} \mid j \in Y_k) \\ 
    &\leq 14.48  \big( \tfrac{\eta}{(\eta - 2\beta) (1-\eta)} \big) \cdot \opt(p) \enspace .
\end{align*}
    Hence,
    \(
        \alg \leq 14.48 \cdot (1+\varepsilon) \cdot \opt(p) 
    \)
    using $\eta = \frac{\varepsilon}{2(1+\varepsilon)}$ and $\beta = \frac{\varepsilon^2}{4(1+\varepsilon)(2+\varepsilon)}$.

\section{A general framework for PSP and malleable job scheduling }\label{sec:framework-psp}

We next present a general framework
for the Polytope Scheduling Problem (PSP) with parsimonious preemption, which encompasses various scheduling problems \cite{ImKM18,JLM25}.
There are $n$ jobs $j=1,\ldots,n$
with processing requirements $p_j$. %
A feasible schedule selects over time for every time $t$ 
a processing rate $y_j(t)$ for each active job $j \in J(t)$ from a given downward-closed polytope $\mathcal P = \{y \in \mathbb R^n_{\geq 0} \mid By \leq \bone \}$ for some matrix $B = [b_{dj}] \in \R^{D \times n}_{\geq 0}$.
Job $j$ completes when it has received $p_j$ amount of processing, that is, its completion time $C_j$ is the earliest point in time $t'$ such that $\int_{r_j}^{t'} y_j(t) \, \mathrm{d}t \geq p_j$.
The goal is to minimize the total completion time $\sum_j C_j$.
We again make the assumption
that all jobs are underpredicted; we can mitigate this assumption as in the previous sections.

\paragraph{The framework.}
Our framework is a generalization of SNAP in (cf.~\Cref{sec:unrelated}).
The main difference lies in Step~4, where we assume to be given a procedure
to convert a rate schedule (a vector belonging to $\mathcal{P}$) into a
schedule with few preemptions. 
\begin{definition}[$\alpha$-approximate $O(1)$-preemptive schedule]\label{def:approximate-preemptive}
	For a given PSP schedule of length $l$ with fixed allocation $\{y_j\}_{j \in J}$,
	an \emph{$\alpha$-approximate $O(1)$-preemptive schedule} satisfies processing requirements $\{l y_j\}_{j \in J}$, has a makespan of at most $\alpha l$, and uses only $O(1)$ preemptions per job.
\end{definition}

The definition of preemption depends on the specific PSP, as each polytope $\mathcal{P}$ determines a distinct scheduling problem. Consequently, Definition~\ref{def:approximate-preemptive} allows us to reason within the abstract setting of the PSP, providing a general framework for our analysis.
We again assume that all $p_j$'s are powers of $1 + \delta$ for some constant $\delta > 0$, which increases the competitive ratio by at most a factor of $1+\delta$.

Our framework generalizes the SNAP algorithm. Thus, we focus here on the key modifications and defer the full description to \Cref{app:framework}. The primary distinction lies in Step 4, which is adapted as follows: we compute an $\alpha$-approximate $O(1)$-preemptive schedule (as per \Cref{def:approximate-preemptive}) for the set of jobs $J_k$. For each job $j \in J_k$, the processing targets are defined as $v_{j,k} = \min\{u_{j,k}, l_k y^{(k)}_j\}$. This schedule starts at time $e_k$ and is guaranteed to finish no later than $e_{k+1} \leq e_k + \alpha l_k$.

\paragraph{Analysis.}
Bounding the number of preemptions is analogous to \Cref{thm:u-preemption-bound} (see \Cref{thm:psp-preemptions}). For the analysis of the 
competitive ratio, for convenience we analyze our algorithm
before the conversion to the $\alpha$-approximate $O(1)$-preemptive
schedule, which we will call the \emph{virtual algorithm}; in this case, we effectively have $e_{k+1} = e_k + l_k$. 
This assumption does not alter the processing each job receives within an epoch. However, the conversion increases each job's completion time by a factor of at most $\alpha$, resulting in a multiplicative loss in the competitive ratio by a factor of $\alpha$.

The primary departure from the analysis of SNAP for unrelated machines is that PSP lacks the super-additivity guarantees for subinstances (cf.~\Cref{lem:decomposition}). Consequently, we must use different techniques to bound the objective. Our main technical novelty is the introduction of a \emph{robust} variant of PSP, which we term $\beta$-robust PSP, and a corresponding analysis of PF within this framework.

\begin{definition}
In the $\beta$-robust PSP problem, we are given a PSP instance and a constant $\beta \in [0,1)$. At any time $t$, an algorithm selects a processing allocation $y(t) \in \mathcal{P}$ for the set of active jobs $J(t)$. An adversary then chooses a subset $N(t) \subseteq J(t)$ of at most $\beta |J(t)|$ jobs whose processing rates are nullified. Formally, the resulting processing allocation is:
    \( \hat{y}_j(t) := y_j(t) \cdot \bone[j \notin N(t)]. \)    
\end{definition}
We show in \Cref{sec:robustpf} 
the following theorem.

\begin{restatable}{theorem}{thmRobustPf}\label{thm:rpf}
	For any $\beta \in [0,1)$, PF is $27/(1-\beta)^3$-competitive for $\beta$-robust PSP.
\end{restatable}

The key observation is that 
the schedule of the virtual algorithm can be interpreted as 
a schedule of PF for $\beta$-robust PSP with a specific choice of the set of nullified jobs.
Using this observation, we show in \Cref{app:framework} that the virtual algorithm achieves a total completion time at most $27\frac{(1+\delta)}{(1 - \beta)^3}$ times the optimum.
Setting $\delta = (1+\varepsilon)(1-\beta)^3 - 1$ and $\beta = 1 - ((1+\delta)/(1+\varepsilon))^{1/3}$ %
yields: %
\begin{theorem}\label{thm:psp-main}
    Given a procedure for computing $\alpha$-approximate $O(1)$-preemptive schedules, there exists a $27\alpha(1+\varepsilon)$-competitive algorithm for PSP 
    that performs at most $O(\frac{1}{\varepsilon^3} \sum_j (1+ \log_2 p_j / \hp_j))$ preemptions for any $\varepsilon > 0$.
\end{theorem}

Finally, we can apply this theorem to malleable job scheduling on unrelated machines.
Details are deferred to \Cref{app:malleable}.

    \begin{restatable}{theorem}{thmMalleableUnrelated}
        \label{thm:m-u}
    For any $\varepsilon > 0$, there is a  $27 \frac{e}{e-1} (1+\varepsilon)$-competitive algorithm that  
    incurs at most $O(\frac{1}{\varepsilon^3} \sum_j (1+ \log_2 p_j / \hp_j))$ preemptions
    for minimizing the total completion time of malleable jobs on unrelated machines with non-decreasing, concave, and non-decreasing work speed-up functions.
    \end{restatable}

\section{Experiments}

To the best of our knowledge, no existing algorithms limit preemption using predictions while simultaneously optimizing for the total completion time objective. To address this gap, we developed and implemented a suite of algorithms, including SNAP. With the exception of the Blind strategy, all of these algorithms incorporate novel algorithmic aspects. We use a default value of $\delta=1$. For additional implementation details, see Appendix~\ref{app:alg-detail}.

\begin{figure}[t]
    \centering
    \includegraphics[width=0.48\textwidth]{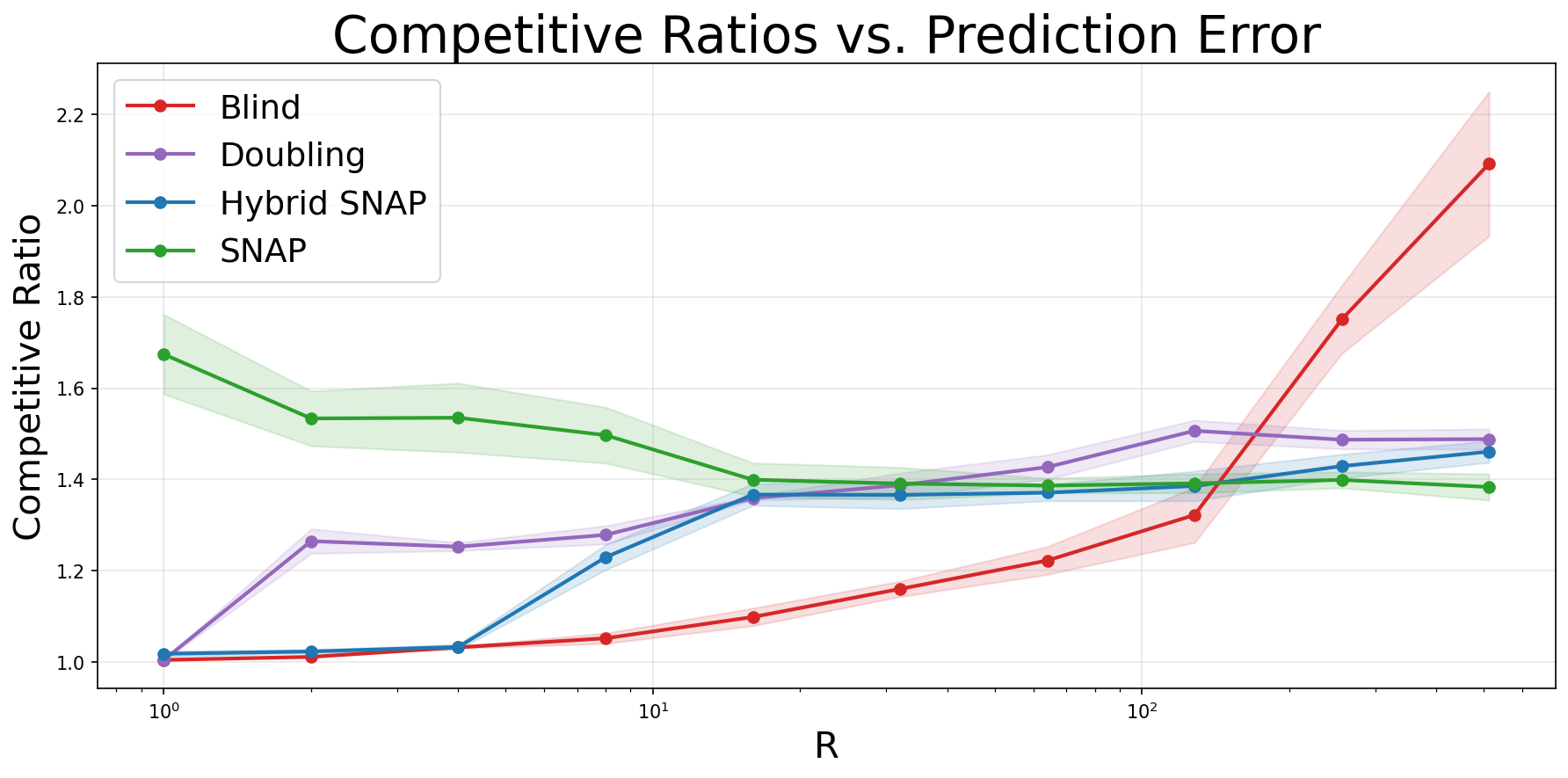}
    \hfill
    \includegraphics[width=0.48\textwidth]{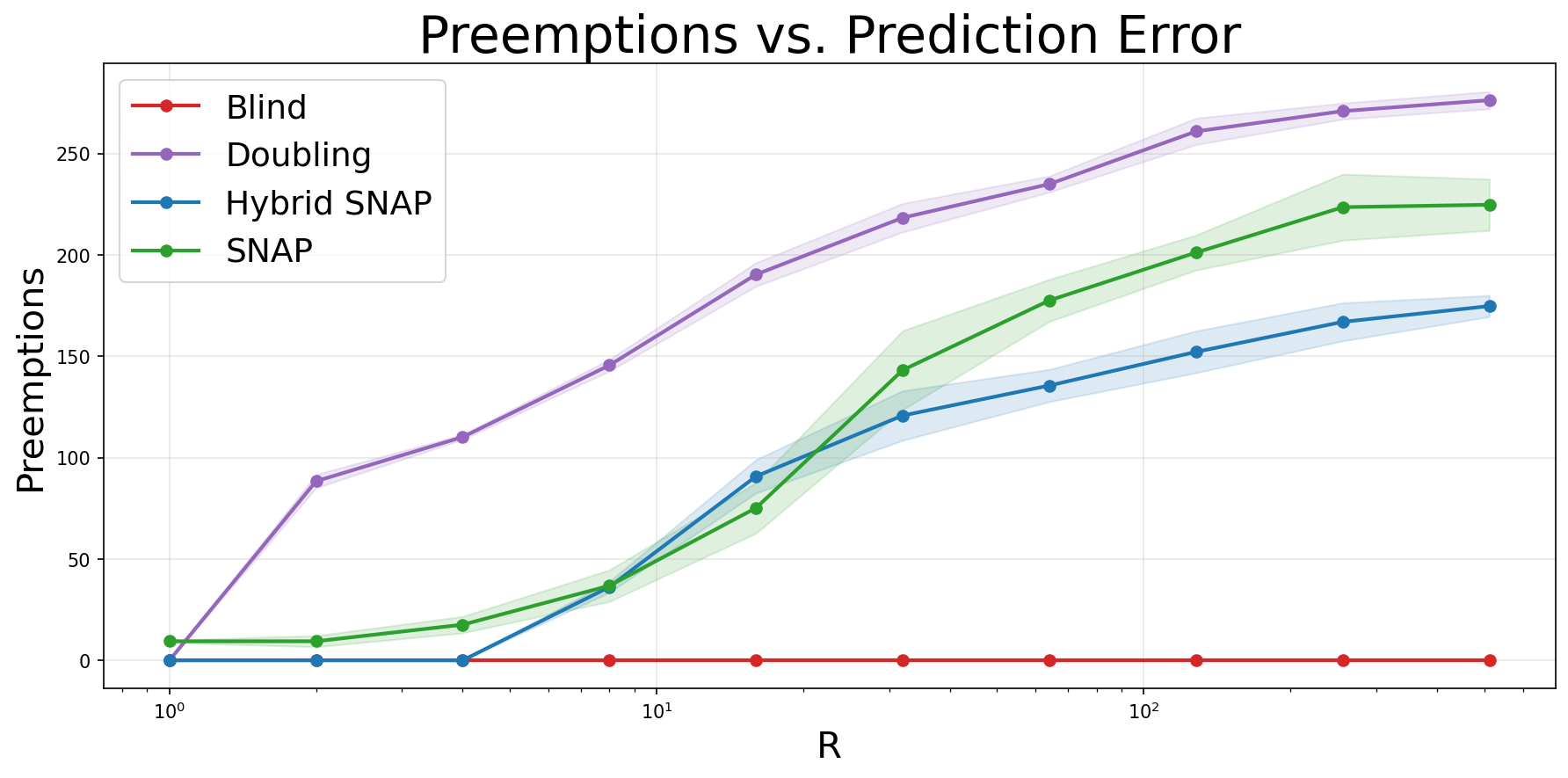}
    \caption{Left: Competitive ratio vs. prediction error parameter $R$. Performance of the Blind strategy degrades significantly as error increases, whereas SNAP remains stable. Right: Average preemptions per job vs. $R$. SNAP exhibits more parsimonious preemption behavior. Hybrid SNAP (with $c=4$) improves upon SNAP for small values of $R$ by leveraging the strengths of the Blind and Doubling.}
    \label{fig:experiments-main}
\end{figure}

\textbf{SNAP:} Our primary algorithm, modified to eliminate idle time by running PMLF on each machine until a $\beta$ fraction of jobs are exhausted. Specifically, we follow Step 4 to assign jobs to machines in an epoch, which are then processed by PMLF on each machine. \textbf{Doubling:} A variant of the immediate-dispatch algorithm in \cite{AnandGK12} (where each machine runs SJF and new jobs are assigned to minimize the marginal increase of the objective). This variant doubles a job's estimated size whenever processing exceeds the current estimate and re-dispatches it to the machine that minimizes the objective increase.
\textbf{Blind:} Based on the immediate-dispatch and non-migratory algorithm in \cite{AnandGK12}, this strategy is executed by assuming $p_j = \hat{p}_j$ for all jobs. This algorithm makes no preemptions.
\textbf{Hybrid SNAP:} A modified SNAP algorithm that begins with Doubling with large milestones $\{ c (1+ \delta)^i \hat p_j \mid  i \geq 1\}$ for a large parameter $c \geq 10$ (to emulate the Blind strategy). Jobs that have reached several milestones are processed by SNAP.

\textbf{Setup:} 
We evaluated the algorithms on synthetic instances with $m=10$ machines and $n=100$ jobs. To model heterogeneity, 20\% of jobs are designated as ``special'' ($p_j \sim \mathcal U[1,200]$) and are restricted to 2 specific machines, while regular jobs ($p_j \sim \mathcal U[1,10]$) can utilize any resource. Intuitively, special jobs represent tasks with large processing requirements that require specific hardware. (For other percentages of special jobs and machines we obtain similar results; see Appendix~\ref{app:exp}). 

The prediction error is controlled via a parameter $R \ge 1$; we set $\hat{p}_j = \lceil p_j/ \xi_j \rceil$ where $\xi_j \sim \mathcal U[1,R]$. We evaluate the algorithms based on the total completion time objective normalized to the optimum (the competitive ratio) and the average number of preemptions per job. Comparisons are performed across exponentially increasing values of $R$. We note that large prediction errors are common in practice \cite{park20183sigma,leis2015good}; in particular, even top-tier database systems routinely misestimate join sizes, which directly dictate execution runtimes, by factors of $1000$ or more.

\textbf{Results:} 
Our experiments demonstrate that SNAP requires considerably fewer preemptions while effectively optimizing total completion time, particularly when prediction errors are large. These findings are summarized in Figure~\ref{fig:experiments-main}. Our main observations are as follows:

\begin{itemize}
    \item Our results show that we can effectively \textbf{leverage the highly preemptive PF algorithm to derive a preemption-parsimonious algorithm} like SNAP. SNAP incurs fewer preemptions than natural benchmarks such as Doubling.
    \item It is known that Blind is 4-competitive if all predictions are accurate~\cite{AnandGK12}, but its competitive ratio degrades linearly with error. Thus, Blind performs well for small $R$ but becomes impractical as $R$ increases. 
    \item To achieve the ``best of both worlds'', we implemented Hybrid SNAP. This approach yields promising results, consistently outperforming or matching others in both the objective value and in minimizing preemptions.
\end{itemize}

Further experimental results, including a detailed analysis of the tradeoff between the competitive ratio and preemptions as a function of $\delta$, are deferred to Appendix~\ref{app:exp}.

\section{Conclusions}

In this work, we presented the first analytical study of the trade-off between job latency minimization and preemption overhead in online scheduling with predictions. By leveraging machine-learned predictions, our SNAP and PMLF algorithms avoid the excessive preemption required by classical non-clairvoyant algorithms. We proved that these algorithms are constant-competitive in various environments, including single, identical, and unrelated machines, with preemption costs that scale logarithmically with the prediction error. Our experiments show that SNAP remains stable in minimizing both preemption overhead and total completion time, and also demonstrate the favorable performance of the other heuristics developed in this study. 

Future research could further explore hybridizing the Blind and SNAP frameworks. This approach seeks to leverage the efficiency of Blind under accurate predictions while maintaining the robustness of SNAP in the presence of large errors, effectively bridging the gap between optimistic performance and worst-case reliability.

\printbibliography

\appendix

\section{Lower bounds}

\subsection{General lower bound on the number of preemptions}

For a fixed algorithm, let $\PM(x)$ denote the maximum number of preemptions that the algorithm makes on a job of size $x$.
The following theorem was originally stated and proven by Motwani et al.~\cite{MPT94}. We give a more detailed proof for completeness and clarity.

\begin{theorem}[\cite{MPT94}]\label{thm:lower-bound-nonclairvoyant}
Consider a non-clairvoyant algorithm that preempts a job of size $x$ at most $\PM(x)$ times.
If $\PM(x) = o(\log x)$, then for every $n \geq 2$ there exists an instance with $n$ jobs on which the algorithm has a competitive ratio of at least $n/2$.
\end{theorem}

\begin{proof}
	Let $B:=2n^2$ and $m=B^{2n}$. 
	Since $\log_m x = \log x / \log m$, 
	there exists a sufficiently large $x_0$ such that 
	for every $x \geq x_0$ we have $\PM(x) \leq \log_m x$, 
	and additionally we can assume that $x_0 = B^{K+1}$ for some integer $K \geq 2$.
	
    Consider an instance on $n$ jobs, each of size $x_0$, and the algorithm's
    schedule at time $t = x_0$. Note that before this time, the algorithm cannot
    complete any job.
	By our assumption, the total number of preemptions until time $t$ can be at most
	\[
	    n \cdot \log_m x_0 
					= n \cdot \log_{B^{2n}} B^{K+1} 
					= n \cdot \frac{K+1}{2n} 
					= \frac{K+1}{2} < K \ ,
	\]
	because $K \geq 2$.
	
	Now, partition the time interval $[1, B^K] \subseteq [0,x_0]$
	into $K$ intervals $I_i = [B^{i-1}, B^i)$ for $i = 1, \ldots, K$.
	Since there are $K$ intervals before time $t$ but less than $K$ preemptions,
	by the pigeonhole principle, there must be at least one interval $I_i$ that contains no preemption.
	Fix such an interval $I_i$ and let $j$ be the job that is processed during $I_i$.
	Note that during $I_i$, job $j$ receives a total processing 
	of at least $B^i - B^{i-1} = (B-1) \cdot B^{i-1} = (2n^2-1) \cdot B^{i-1}$.
	Let $e_{j'}(B^i)$ denote the amount of processing that job $j'$ received until time $B^i$.
	Note that $e_{j'}(B^i) \geq (2n^2-1) \cdot B^{i-1}$, and for every job $j' \neq j$ we have $e_{j'}(B^i) \leq B^{i-1}$.
	
	We next define a new instance with the same set of jobs, but with processing times
	\begin{itemize}[nosep]
		\item $p_j = e_j(B^i)$ and
		\item $p_{j'} = e_{j'}(B^i) + \varepsilon$ for all $j' \neq j$, where $\varepsilon = \frac{B^{i-1}}{2n^2}$.
	\end{itemize}
	By construction and the fact that the algorithm is deterministic,
	observe that no job completes before time $B^i$.
	Moreover, job $j$ completes exactly at time $B^i$.
	
	We compute a lower bound on the competitive ratio of the algorithm on this instance.
	Note that $C_j \geq C_{j'}$ for all $j' \neq j$. Hence, 
	\[
	    \alg \geq n C_j \geq n \cdot p_j \ .
	\]
	We next give an upper bound on the optimal objective value.
	Note that for every job $j' \neq j$, we have
	\[
		n^2 p_{j'} \leq n^2 \left(B^{i-1} + \frac{B^{i-1}}{2n^2}\right) = n^2 B^{i-1} + \frac{B^{i-1}}{2} \leq (2n^2-1) \cdot B^{i-1} \leq p_j \ .
	\]
	Let $p = p_j / n^2$.
	Scheduling job $j$ last achieves a total completion time of at most
	\[
	    \frac{(n-1)n}{2}p + (n-1)p + p_j 
					\leq n^2 p + p_j \leq 2p_j  \ .
	\]
	Thus,
	\[
	    \frac{\alg}{\opt} \geq \frac{n \cdot p_j}{2p_j} = \frac{n}{2} \ ,
	\]
	which completes the proof.
\end{proof}

\begin{theorem} \label{thm:lb2}
	If for each job $j$ it holds that $\hp_j \leq p_j \leq \lambda_j \hp_j$
	and there exists a job $j$ such that the algorithm preempts $j$ at most $o(\log \lambda_j)$,
	then the algorithm has a competitive ratio of at least $\Omega(n)$ for sufficiently large $\lambda$ and $n$.
\end{theorem}

\begin{proof}
	Follows directly from the proof of \Cref{thm:lower-bound-nonclairvoyant} by
	setting $\hp_j = 1$ for all jobs $j$. This makes the jobs indistinguishable and the proof strategy applicable.
    Moreover, $\lambda_j = p_j / \hp_j = p_j$ for all jobs $j$,
    hence the guarantees translate.
\end{proof}

\subsection{Trade-off lower bound}

\begin{theorem}
	If a deterministic learning-augmented algorithm uses at most $s$ preemptions in total for 
	a correctly predicted instance with $n$ jobs, then there exists an instance where every job is 
	overestimated and the algorithm has a competitive ratio of at least $\Omega((n-s)/s)$.
\end{theorem}

\begin{proof}
	Let $\hp_j = p_j = P$ for all jobs $j \in J$ be the correctly predicted instance with $n$ jobs. Consider the schedule of the fixed deterministic algorithm on this instance.
	Until time $P$, the algorithm cannot have completed any job. 
	Assume that until time $P$, the algorithm does at most $s$ preemptions, hence the algorithm 
	can touch at most $s+1$ jobs until time $P$.
	Let $Z$ denote the set of jobs that were not touched by the algorithm until time $P$,
	and note that $|Z| \geq n - (s+1)$.
	For each job $j \in J \setminus Z$, denote by $q_j$ the amount of work done by the algorithm on $j$ until time $P$.
	
	We now construct another instance as follows. We set $\hp_j = P$ for all $j \in J$ and
	\[
	    p_j = \begin{cases}
						   1 & \text{if } j \in Z \\
						   q_j+\varepsilon & \text{if } j \in J \setminus Z
					\end{cases}
	\]
	Since the algorithm is deterministic, it behaves the same on this new instance as on the original instance, because with
	the same schedule as before it again will not complete any job until time $P$.
	Thus, $\alg \geq P \cdot |Z|$ for this instance.
	However, an optimal solution can schedule all short jobs in $Z$ first, and then schedule the long jobs in $J \setminus Z$ in any order.
	This gives $\opt \leq \frac{|Z| (|Z| + 1)}{2} + (n - |Z|) (|Z| + P + n\varepsilon)$.
	By choosing $P \geq n^3$,
	we conclude that
	\[
	    \frac{\alg}{\opt} \geq \frac{P \cdot |Z|}{\frac{|Z| (|Z| + 1)}{2} + (n - |Z|) (|Z| + P + n\varepsilon)}
		\geq \frac{P \cdot (n-(s+1))}{ n^2 + (s+1) (n + P + n\varepsilon)}
		\geq \Omega \left( \frac{n-s}{s} \right) \ .
	\]
	This completes the proof of the theorem.
\end{proof} 

\section{Single machine} \label{app:single-machine-appendix}

\subsection{Analysis of PMLF}

We present the full proof of Theorem~\ref{thm:single}.

\thmSingle*

\begin{proof}
    We first bound the number of preemptions. Note that $j$ is only preempted when it is moved between queues.
    Since job $j$ starts in $Q_{k'}$ where $(1+\delta)^{k'} \leq \hat p_j < (1+\delta)^{k'+1}$ 
    and completes in $Q_{k}$ where $(1+\delta)^{k} \leq p_j < (1+\delta)^{k+1}$, 
    it is moved at most $k - k' + 1 = O(\log_{1+\delta}(p_j / \hat p_j)) = O(\frac{1}{\delta} \log_{2}(p_j / \hat p_j))$ times across queues. 
	
	We next prove an upper bound on the competitive ratio, for which we 
	use a classic delay decomposition~\cite{MPT94}.
	Fix a job~$j$ with $(1+\delta)^k \leq p_j < (1+\delta)^{k+1}$ and let $C_j$ 
    denote its completion time. Note that at the completion time $C_j$, it belongs to queue~$Q_k$.
	Let $d(i,j)$ denote the total delay that job $i$ causes to job $j$, %
    which is the total amount of work done on $i$ before time $C_j$.
	Clearly, $d(i,j) \leq p_i$.
	Moreover, since $j$ completes at time $C_j$ while belonging to queue $Q_k$, the algorithm can only work on $i$ before $C_j$ whenever $i$ belongs to a queue $Q_{k'}$ with $k' \leq k$. This
	implies $d(i,j) < (1+\delta)^{k+1}$.
	Thus, $d(i,j) \leq \min\{p_i, (1+\delta)^{k+1} \}$, which is at 
	most $(1+\delta) \min\{p_i, p_j\}$ because $p_j \geq (1+\delta)^k$.
	Since our algorithm does not idle, %
	\[ 
	    C_j \leq \sum_{i=1}^n d(i,j) \leq (1+\delta) \sum_{i=1}^n \min\{p_i, p_j\} \ .
	\]
	Summing over all jobs $j$ gives
	\begin{equation*} 
		\sum_{j=1}^n C_j \leq (1+\delta) \sum_{j=1}^n \sum_{i=1}^n \min\{p_i, p_j\} \leq (1+\delta) \cdot 2  \opt \ ,
	\end{equation*}
	where the second inequality follows the fact that scheduling jobs in non-decreasing order of their processing times is optimal~\cite{smith1956various}, thus, $\opt = \sum_j p_j + 
\sum_{j=1}^n \sum_{i>j}^n \min\{p_i, p_j\}$.    
\end{proof}

\subsection{Job arrival over time}

We next show how PMLF and its analysis in the proof of \Cref{thm:single} can be
generalized to the setting with online arrival of jobs.
In this setting, a job $j$ arrives at its \emph{release time} $r_j \geq 0$,
and an online algorithm does not know about $j$'s existence before that time.
PMLF can be extended in a straightforward way to this dynamic scenario:
\begin{itemize}
    \item Maintain a set of FIFO queues $\{Q_0,Q_1,\ldots\}$.
    \item Whenever a job $j$ with prediction $\hat p_j$ arrives,
    we push it into queue $Q_k$ where $(1+\delta)^k \leq \hat p_j \leq (1+\delta)^{k+1}$.
    \item At any time $t$, work on the front job of the non-empty
    queue of the lowest index.
    \item If a job $j \in Q_i$ has received a total processing equal to $(1+\delta)^{i+1}$, move it from $Q_i$ to the end of $Q_{i+1}$.
\end{itemize}

We prove the following theorem.

\begin{theorem}\label{thm:single-machine-release-times}
    For any constant $\delta > 0$, PMLF is $(3+2\delta)$-competitive 
    for minimizing the total completion time on a single machine
    with online job arrival and performs at most $O(\frac{1}{\delta} \log_2(p_j / \hat p_j))$ preemptions for each job $j$ if there are only underpredictions.
\end{theorem}

\begin{proof}
    The proof for the bound on the number of preemption is the same
    as for \Cref{thm:single}.

    We next prove an upper bound on the competitive ratio, for which we 
	use a classic delay decomposition~\cite{MPT94}.
	Fix a job~$j$ with $(1+\delta)^k \leq p_j < (1+\delta)^{k+1}$ and let $C_j$ 
    denote its completion time. Note that at the completion time $C_j$, it belongs to queue~$Q_k$.
	Let $d(i,j)$ denote the total delay that job $i$ causes to job $j$ after $j$ has been released, 
    which is the total amount of work done on $i$ between times $r_j$ and $C_j$.
	Clearly, $d(i,j) \leq p_i$.
	Moreover, since $j$ completes at time $C_j$ while belonging to queue $Q_k$, the algorithm can only work on $i$ during $[r_j,C_j]$ whenever $i$ belongs to a queue $Q_{k'}$ with $k' \leq k$. This
	implies $d(i,j) < (1+\delta)^{k+1}$.
	Thus, $d(i,j) \leq \min\{p_i, (1+\delta)^{k+1} \}$, which is at 
	most $(1+\delta) \min\{p_i, p_j\}$ because $p_j \geq (1+\delta)^k$.
	Since our algorithm does not idle unnecessarily, we have
	\[ 
	    C_j 
        \leq r_j + \sum_{i=1}^n d(i,j) 
        \leq r_j + (1+\delta) \sum_{i=1}^n \min\{p_i, p_j\} \ .
	\]
	Summing over all jobs $j$ gives
	\begin{equation*} 
		\sum_{j=1}^n C_j 
        \leq \sum_{j=1}^n r_j + (1+\delta) \sum_{j=1}^n \sum_{i=1}^n \min\{p_i, p_j\} 
        \leq \opt + (1+\delta) \cdot 2  \opt \ ,
	\end{equation*}
	where the second inequality uses
    \begin{itemize}
        \item that
    scheduling jobs in non-decreasing order of their processing times is optimal~\cite{smith1956various}, thus, $\opt = \sum_j p_j + 
\sum_{j=1}^n \sum_{i>j}^n \min\{p_i, p_j\}$, and
    \item $\sum_{j=1}^n r_j \leq \opt$ is a trivial lower bound.
    \end{itemize}
    This completes the proof of the theorem.
\end{proof}

\subsection{Weakly enforcing underestimations}

In this section, we prove \Cref{thm:single-weakly-enforce}, which we first restate for convenience.

\thmSingleEnforce*

\begin{proof}
Let $J^u$ and $J^o$ be the set of overestimated and underestimated jobs. Note that $|J^o| \leq g$.
Let $n^u(t)$ be the number of underestimated unfinished jobs at time $t$, and $n(t)$ be the total number of unfinished jobs at time $t$.
Thus, $n(t) \geq g/\gamma \geq (n(t) - n^u(t)) / \gamma$ for every
time $t \leq T$.
Now we can compute
\begin{align*}
	\alg = \int_0^\infty n(t) \, dt &= \int_{0}^T n^u(t) \, dt + \int_{0}^T (n(t)-n^u(t)) \, dt + \int_{T}^\infty n(t) \, dt \\
	&\leq \int_{0}^T n^u(t) \, dt + \int_{0}^T \gamma n(t) \, dt + \int_{T}^\infty n(t) \, dt \\
	&\leq \int_{0}^T n^u(t) \, dt + \gamma \alg + \int_{T}^\infty n(t) \, dt \ .
\end{align*}
Hence,
\begin{align}
(1-\gamma) \alg 
&\leq \int_{0}^T n^u(t) \, dt + \int_{T}^\infty n(t) \, dt \notag \\
&\leq \sum_{j \in J^u} C_j + \sum_{j \in J^o} (C_j - T) \ ,
\label{eq:single-overestimated}
\end{align}
where the last inequality uses the fact that each job $j \in J^u$ is counted in both $n^u(t)$ and $n(t)$, and each job $j \in J^o$ is only counted in $n(t)$ and not in $n^u(t)$.
We can now consider both sums individually.

For every underestimated job $j \in J^u$, we can use the same proof as in the proof of \Cref{thm:single}. Before $j$'s
completion, it can only be delayed by jobs $i$ that are in queues of
index at most $\lfloor \log_{1+\delta} p_j \rfloor$, hence $i$
delays $j$ by at most $\min\{p_i, (1+\delta)p_j\}$ many times.
Thus, 
\[
C_j \leq \sum_{i \in J} \min\{p_i, (1+\delta)p_j\} \ .
\]

For every overestimated job $j \in J^o$, we can also 
use the same argumentation but only for delays that happen after 
time $T$. This gives 
\[
C_j - T \leq \sum_{i \in J} \min\{p_i, (1+\delta)p_j\} \ .
\]

Combining both bounds implies that \eqref{eq:single-overestimated} is at most
\[
    (1+\delta) \sum_{j \in J} \sum_{i \in J} \min\{p_i, p_j\}
    \leq (1+\delta) 2 \cdot \opt \ .
\]
Thus, $\alg \leq \frac{1+\delta}{1-\gamma}\cdot 2 \cdot \opt = 2(1+\varepsilon) \cdot \opt$ by choosing $\delta = \frac{\varepsilon}{2}$ and $\gamma = \frac{\varepsilon}{2(1+\varepsilon)}$.

We next bound the number of preemptions.
For underpredicted jobs the same bound as in \Cref{thm:single}
applies for all preemptions before time $T$.
At time $T$, each unfinished job is reinserted at most once,
and after time $T$, each job is preempted at most $O(\log_{1+\delta} p_j)$.
Hence, the total number of preemptions is at most
\[
O \bigg( \frac{1}{\delta} \sum_{j \in J} 
\log \frac{p_j}{\hp_j} + \frac{g}{\gamma \delta} \log p_{\max} \bigg) \ ,
\]
which implies the stated bound for our choice of $\delta$ and $\gamma$.
\end{proof} 

\section{Scheduling on identical parallel machines}\label{sec:identical-machines}

In this section, we show that our 
single machine algorithm PMLF can be extended to scheduling 
on $m$ parallel identical machines.
We first restate the generalization of PMLF for identical parallel machines.
Let $\delta > 0$ be a constant. 
\begin{itemize}
    \item Maintain a set of queues $\{Q_0, Q_1, \ldots \}$.
   Initially, place a job $j$ with $(1+\delta)^k \leq \hat p_j < (1+\delta)^{k+1}$ into $Q_k$.
	\item At any time $t$, for each available job $j$, let $(k,a)$ be the tuple where $k$ is the index of $j$'s current queue, and $a$ is its position in that queue, where the head of the queue has position $1$. Let $\preceq$ denote the total over all available jobs that corresponds to the lexicographic order of those tuples.
    Schedule the first at most $m$ jobs of $\preceq$.
	\item If a job $j \in Q_i$ has received a total processing equal to $(1+\delta)^{i+1}$ at time $t$, remove it from $Q_i$ and place it at the end of $Q_{i+1}$.
\end{itemize}

We show the following theorem, which matches the guarantees of our single machine result given in \Cref{thm:single}.
\begin{theorem}
    \label{thm:identical-parallel}
    For the setting of only underestimated predictions, 
PMLF is $(2+2\delta)$-competitive for minimizing the total completion time on parallel identical machines and 
performs at most $O(\frac{1}{\delta} \sum_j \log(p_j / \hat p_j))$ preemptions.
\end{theorem}

We split the proof of this theorem into two lemmas.
We first bound the number of preemptions, which is equivalent to 
the argument in the proof of \Cref{thm:single}.
\begin{lemma}\label{lem:identical-preemption}
    For the setting of only underestimated predictions, PMLF uses at most $O(\frac{1}{\delta} \sum_j \log(p_j / \hat p_j))$ preemptions on $m$ parallel identical machines.
\end{lemma}

\begin{proof}
We first argue about the number of preemptions, which is equivalent to 
the argument in the proof of \Cref{thm:single}.
Note that $j$ is only preempted when it is moved between queues.
Since job $j$ starts in $Q_{k'}$ where $(1+\delta)^{k'} \leq \hat p_j < (1+\delta)^{k'+1}$ 
and completes in $Q_{k}$ where $(1+\delta)^{k} \leq p_j < (1+\delta)^{k+1}$, 
it is moved at most $k - k' + 1 = O(\log_{1+\delta}(p_j / \hat p_j))$ often. 
\end{proof}

For proving the bound on the competitive ratio,
we require the following mixed lower bound.

\begin{lemma}[Lemma~9 in \cite{BBEM12}]\label{lem:mixed-lower-bound}
    For the problem of minimizing the total completion time of $n$ jobs on $m$ identical parallel machines, given a partition $p_j = p_j^{(1)} + p_j^{(2)}$ for each job $j$,
    it holds that
    \[
        \opt \geq \frac{1}{2m} \sum_{j=1}^n \sum_{i=1}^n \min\{p^{(1)}_i, p^{(1)}_j\} + \sum_{j=1}^n p_j^{(2)} \ .
    \]
\end{lemma}

We now prove that the algorithm is $(2+2\delta)$-competitive.

\begin{lemma}\label{lem:identical-uniform}
    For the setting of only underestimated predictions,
    PMLF is $(2+2\delta)$-competitive for minimizing the total completion time on parallel identical machines.
\end{lemma}

\begin{proof}
Note that since there are no release dates, the number of available jobs cannot increase over time. 
Let $\tau$ be the first time in the schedule when there are less than $m$ jobs available.
Note that after time $\tau$, there is no time where at least $m$ jobs are available,
and before time $\tau$, there is no time when a machine is idling.

Fix a job $j$.
Let $p_j^{(1)}$ and $p_j^{(2)}$ be the amounts of processing that $j$ receives until time $\tau$ and after time $\tau$, respectively. ($p_j^{(2)}=0$ if $j$ completes before time $\tau$.)
Note that $p_j = p_j^{(1)} + p_j^{(2)}$.

Let $k_j$ be the highest queue index that $j$ reaches in the schedule before time $\tau$.
Note that for a total time of $\min\{C_j, \tau\}$, the algorithm is busy processing jobs in $Q_0,\ldots,Q_{k_j}$ on exactly $m$ machines.
Since each job $i$ can be processed by at most $\min\{p^{(1)}_i, (1+\delta)^{k_j+1} \}$ amount of time in these queues before time $\tau$, we have
\[
m \cdot \min\{C_j, \tau\}
\leq \sum_{i = 1}^n \min\{p^{(1)}_i, (1+\delta)^{k_j+1}\} 
\leq \sum_{i = 1}^n \min\{p^{(1)}_i, (1+\delta)p^{(1)}_j\} 
\leq (1+\delta) \sum_{i = 1}^n \min\{p^{(1)}_i, p^{(1)}_j\} \ . 
\]
Since $C_j \leq \min\{C_j, \tau \} + p^{(2)}_j$, we have
\begin{align*}
    \alg &\leq \sum_j C_j \leq \frac{1+\delta}{m} \sum_{j=1}^n \sum_{i = 1}^n \min\{p^{(1)}_i, p^{(1)}_j\} + \sum_{j=1}^n p_j^{(2)} \\ 
    &\leq 2(1+\delta) \bigg( \frac{1}{2m} \sum_{j=1}^n \sum_{i = 1}^n \min\{p^{(1)}_i, p^{(1)}_j\} + \sum_{j=1}^n p_j^{(2)} \bigg) \leq 2(1+\delta) \opt \ ,
\end{align*}
where the final inequality is due to \Cref{lem:mixed-lower-bound}.
\end{proof}

\section{Scheduling on unrelated machines}
    \label{sec:u-over}

\subsection{Deferred proofs from \Cref{sec:unrelated}} \label{app:u-deferred}
  
\claimSnapCombination*

\begin{proof}
Consider schedules $S^*_J$, $S^*_X$, and $S^*_Y$ that correspond to 
$\opt(l_k y_j^{(k)} \mid j \in J_k)$, $\opt(l_k y_j^{(k)} \mid j \in X_k)$, 
and $\opt(l_k y_j^{(k)} \mid j \in Y_k)$, respectively. Further, let $C_{j, X}$ 
and $C_{j, Y}$ be job $j$'s completion times in $S^*_X$ and $S^*_Y$, respectively.

To establish the inequality, it suffices to construct a schedule where each job 
$j \in J_k$ with size $l_k y_j^{(k)}$ completes by time $C_{j, X}/\eta$ 
if $j \in X_k$ and by time $C_{j, Y}/(1-\eta)$ if $j \in Y_k$. Towards this end, we ``slow down'' $S^*_X$ by a factor of $\eta \in (0,1)$ and 
$S^*_Y$ by a factor of $1 - \eta$. Formally, let $z_{j, t, X}$ and $z_{j, t, Y}$ 
be job $j$'s processing rates at time $t$ in $S^*_X$ and $S^*_Y$, 
respectively. Then, $j$'s processing rate at time $t$ in the new schedule is defined as:
\[
\eta z_{j, \eta t, X} + (1 - \eta) z_{j, (1 - \eta)t, Y}.
\]
It is straightforward to check that this is a valid schedule since the set of 
feasible processing rates forms a polytope (and is therefore convex), and 
$z_{j, \eta t, X}$ and $z_{j, (1 - \eta)t, Y}$ are both feasible. Further, we have $\int_{t = 0}^{C_{j, X}/\eta} \eta z_{j, \eta t, X} dt = \int_{\tau = 0}^{ C_{j, X}}  z_{j, \tau, X} d \tau$. This means that by time $C_{j, X}/\eta$, $j$ has been processed in the new schedule as much as it has been by time $C_{j, X}$ in schedule $S^*_X$. Therefore, the new schedule completes job $j \in X_k$ by time $C_{j, X}/\eta$ as desired. We can similarly show that $j \in Y_k$ completes by time $C_{j, Y}/(1-\eta)$. This completes the proof. 
\end{proof}

\lemepochbound*

\begin{proof}
    Since PF can complete each job of size $l_k y_j^{(k)}$ in $X_k \subseteq J_k$ within  $l_k$ time steps, 
    we have 
    \begin{equation}\label{eq:bound-x}	
        \opt(l_k y_j^{(k)} \mid j \in X_k) \leq l_k |X_k| \ .
    \end{equation}
    Combining Claims~\ref{claim:u-1} and \ref{claim:u-2} with %
    $|X_k| \leq \beta n_k$, we obtain
    \begin{align*}
	&n_k l_k \\ 
	&\leq 2 \left( \frac{1}{\eta} \opt(l_k y_j^{(k)} \mid j \in X_k) + \frac{1}{1-\eta} \opt(l_k y_j^{(k)} \mid j \in Y_k) \right) \\
	&\leq 2 \left( \frac{1}{\eta} l_k |X_k| + \frac{1}{1-\eta} \opt(l_k y_j^{(k)} \mid j \in Y_k) \right) \\
	&= 2 \left(  \frac{\beta}{\eta} l_k n_k + \frac{1}{1-\eta} \opt(l_k y_j^{(k)} \mid j \in Y_k) \right) \ .
    \end{align*}
    By rearraging terms, we obtain the desired inequality.
\end{proof}

\subsection{Job arrival over time: sketch}

We now discuss how to extend our algorithm and analysis for unrelated machines to accommodate jobs arriving over time. We note that we do not optimize the parameters introduced here, specifically $\nu$. For convenience we will set $\nu = 2\beta$. Furthermore, while we detail the necessary algorithmic changes and address the subtle points of the analysis, we provide only a compact proof sketch. As discussed in the introduction, this follows our primary focus, which remains on settings where jobs do not arrive over time.

Algorithmically, since our approach simulates the PF schedule within each epoch using a non-preemptive schedule, we should refine epochs whenever a significant number of new jobs arrives. Specifically, we initiate a sub-epoch if the number of new arrivals reaches a $\nu$-fraction of the currently active jobs: if $J(e_k)$ denotes the set of jobs present at the beginning of epoch $k$, the first sub-epoch (if it exists) begins at time $e_{k, 1}$ when $\nu |J(e_k)|$ new jobs have arrived while the non-preemptive PF-simulated schedule (Step 4) is being executed. In general, if $J(e_{k, h})$ are the jobs alive at the start of the $h$th sub-epoch, then the $(h+1)$th sub-epoch begins at time $e_{k, h+1}$ when $\nu |J(e_{k, h})|$ new jobs have arrived while the non-preemptive PF-simulated schedule (Step 4) of the $h$th sub-epoch is being executed. 

We note that from an algorithmic standpoint, epochs and sub-epochs are identical; this distinction is maintained primarily for the purpose of analysis.

\paragraph{Bounding preemptions.} We first discuss how we can upper bound the number of preemptions. If an epoch $k$ has no sub-epochs, then the analysis proceeds as before as the newly arriving jobs are not preempted. So, suppose it has sub-epochs, $1, 2, \ldots, h, \ldots, H$. The additional preemptions initiated at the beginning of each sub-epoch, that is at time $e_{k, h}$, can be charged to the new jobs arriving during $(e_{k, h-1}, e_{k, h}]$. Here, for convenience, $e_{k, 0} := e_k$. Note that because each arrival is responsible for at most $1 / \nu$ jobs in this charging scheme and is charged only once during its lifetime, the scheme is valid.

\paragraph{Bounding total completion time.}
We now focus on bounding the total completion time objective. To this end, consider a fixed epoch $k$. First suppose it contains no sub-epochs. Recall that the cost of each epoch is bounded by $4 n_k l_k$, where $n_k := |J(e_k)|$ and $4l_k$ represents an upper bound on the epoch's length. Since the additional arrivals (up to $\nu n_k$ jobs) cannot be processed until time $e_k$, the optimal solution must also wait, just as our algorithm does, thereby incurring a comparable cost. During the epoch, the additional cost incurred is at most $4 \nu n_k l_k$. Consequently, this increases the objective by a factor of at most $1 + \nu$.

We now consider the more interesting case where the fixed epoch $k$ contains sub-epochs. For easy notation, let $t_h := e_{k, h}$ and $t_0 := e_{k}$. Similarly, $J_0 := J(e_k)$, and $J_h := J(e_{k, h})$. Also, for simplicity, we call the time interval $[t_0, t_1)$ as the $0$th epoch. Let $t_{H+1}$ denote the last epoch's end time. 

Since we set $\nu = 2\beta$, less than $\nu / 2$ fraction of jobs reach their checkpoints, meaning that at most $\nu /2$ fraction of jobs can be completed; yet $\nu$ fraction of jobs newly arrive. Thus, it follows that $|J_{h+1}| \geq (1 + \nu / 2) |J_{h}|$ for all $h = 0, 1, \ldots, H-1$. We charge the cost of jobs that were alive at any point during $[e_0, e_{H-2})$, denoted as the set $V_1$, up until time $t_{H-1}$ to the arrival times of the jobs arriving in the penultimate epoch, $H-1$, which we denote as $V_2$. To see this, observe that $|V_2|$ is at least $\nu |J_{H-1}|$ and they all arrive no earlier than $t_{H-1}$. Furthermore, the cost of $V_1$ until $t_{H-1}$ is at most $t_{H-1} (1+ \nu) |J_{H-1}| \leq t_{H-1} (1 + \nu) |V_2| (\frac{1}{\nu} )$. So, we will have a competitive ratio of $\frac{1}{\nu} (1 + \nu)$ in this charging.

Before bounding the waiting time for jobs in the final two sub-epochs, we discuss a subtlety in the analysis. Suppose the majority of jobs in $V_2$ arrive shortly after $t_{H-1}$. In this case, these arrivals might not have a sufficiently large "arrival time window" to offset the cost incurred by jobs in $J_{H-1}$ up to $t_H$. (The cost of $J_H$ through $t_{H+1}$ can be easily handled, similar to the case with no sub-epochs or only the $0$th epoch). However, in sub-epoch $H-1$, we simulated the PF schedule based on the jobs in $J_{H-1}$ while ignoring $V_2$. If we had to preempt sub-epoch $H-1$ after simulating very few jobs up to what PF would have done, our non-preemptive simulation of PF in that interval may not be faithful to the actual PF schedule, meaning that the simulation time was effectively "wasted."

\noindent The remaining goal is to upper bound the waiting cost $W$ of $J_{H-1}$ accumulating during the interval $[t_{H-1}, t_H)$ (including $V_2$, the waiting cost will be a factor of $1+\nu$ larger). To upper bound $W$, we consider several cases; in some cases, we will charge $W$ to specific subsequent epochs.

\begin{enumerate}
    \item The last sub-epoch $H$ of epoch $k$ completes at least $(\nu / 4) |J_H|$ jobs. We charge $W$ to the total completion time of these jobs, which increases the competitive ratio by a factor of $\Theta(1 / \nu)$. This charging is done ``locally'' within the same epoch, ensuring no overcharging. 

    \item Otherwise, let $k' > k$ be the latest epoch such that at least $(\nu / 4) |J_H|$ jobs are completed during epochs $k+1, k+2, \ldots, k'$. 

    \begin{enumerate}
        \item If any of the epochs $k+1, k+2, \ldots, k'-1$ has a sub-epoch---say $k''$ is the earliest epoch having a sub-epoch---then there must exist at least $\nu (1 - \nu / 2) |J_H|$ jobs that arrive during epoch $k''$. Each of these jobs has an arrival time greater than the waiting time of each job accounting for $W$ during sub-epoch $H-1$ of epoch $k$. Thus, we can charge $W$ to the total arrival time of these jobs, which increases the competitive ratio by at most $O(1 / \nu)$ additively. Note that this cross-epoch charging occurs to them only once, ensuring no overcharging. 

        \item Otherwise, we charge $W$ to the total completion time of jobs completed during epochs $k+1, \ldots, k'$. Note that the waiting cost $W$ occurs during epoch $k$ (which has a sub-epoch) and it is charged to some of consecutive subsequent epochs that have no sub-epochs. Therefore, overcharging does not occur. 
    \end{enumerate}
\end{enumerate}

This completes the sketch of the analysis.

\subsection{When predictions include overestimation}
In this section, we show how to extend our results to cases where predictions include overestimation, i.e., $\hat{p}_j > p_j$ for some jobs $j$. We consider two strategies to handle this.

\subsubsection{Strictly enforcing underestimation}
The first approach is to ensure that no overestimation is effectively perceived by the algorithm. Specifically, if a job completes after receiving less than $\hat{p}_j$ units of processing, we treat the job as remaining active until it has received exactly $\hat{p}_j$ units. This is equivalent to processing jobs where each job $j$ has an effective size of $\max\{p_j, \hat{p}_j\}$ and a predicted size of $\hat{p}_j$. Under this strategy, the bound in Theorem~\ref{thm:u-preemption-bound} becomes:
\[
O \bigg( \frac{1}{\beta \delta} \sum_j \log \bigg( \frac{\max\{p_j, \hat{p}_j\}}{\hat{p}_j} + 1 \bigg) \bigg) = O \bigg( \frac{1}{\beta \delta} \sum_j  \log \bigg( \frac{p_j}{\hat{p}_j} + 1 \bigg) \bigg).
\]
Thus, the asymptotic bound for preemptions remains unchanged. For the total completion time objective, this modification increases each job's effective size by a factor of at most $u := \max_j \hp_j / p_j$. Since increasing every job size by a factor of $u$ increases the optimal objective by at most the same factor, the competitive ratio becomes $14.48(1+\varepsilon)u$.

\subsubsection{Weakly enforcing underestimation} 
Since overprediction has a significantly negative effect on the competitive ratio, we can intentionally underpredict job sizes at the cost of additional preemptions. If the number of overpredicted jobs is small and we have an upper bound $g$ on this quantity with high probability (w.h.p.), we can handle these cases more effectively. Intuitively, predicting a single global parameter—the total number of overpredicted jobs—is easier than identifying specifically which jobs are overpredicted.

Suppose we know that there are at most $g$ \emph{jobs with overpredicted sizes}. We run the SNAP algorithm as before until the number of alive jobs drops to $g / \varepsilon$ or less. More specifically, at time $e_k$, if $n_k \leq g / \varepsilon$, we reset the predicted size of every alive job to the smallest power of $1+\delta$ that is no smaller than its current cumulative processed size. This effectively means that from this point forward, we ignore the original predictions to ensure we only deal with \emph{underpredictions}. Let this transition occur at the beginning of epoch $K'+1$. Another small caveat in running SNAP until epoch $K'$ is if a job turns out to be overestimated. So, a job may complete before it receives $v_{j, k}$  units of processing in epoch $k$ but this can only reduce the current epoch's length, which is only good for the algorithm.

\paragraph{Competitive ratio analysis.}
The analysis then proceeds similarly. The primary difference is that until epoch $K'$, we focus only on the underestimated jobs. That is, $n^u_k$ denotes the number of \emph{underestimated} jobs alive at $e_k$. Lemma~\ref{lem:u-alg-virt-bound} is modified as follows:

\begin{lemma}[Modification of Lemma~\ref{lem:u-alg-virt-bound}]
\label{lem:alg-virt-bound-m}
$\alg \leq \sum_{k=1}^{K'} (n^u_k + g) 4 l_k + \sum_{k=K'+1}^{K} n^u_k  4 l_k 
    \leq 4 (1 + \varepsilon) \sum_{k=1}^{K} n^u_k   l_k.$
\end{lemma}

Here, note that since epoch $K'+1$, all jobs are effectively underestimated (in the analysis of the completion time objective only), thus, $n_k = n^u_k$ for all $k \geq K'+1$. 

The subsequent claims, lemmas, and corollaries need only very minor changes. Let $J^u_k$ be the subset of jobs in $J_k$ that are underestimated. We analogously restrict $X_k$ and $Y_k$ to $X^u_k$ and $Y^u_k$, respectively. 

\begin{claim}[Modification of Claim~\ref{claim:u-1}]
    \label{claim:u-1-m}
For each epoch $k$,     
	we have $n^u_k l_k \leq 2 \cdot \opt(l_k y_j^{(k)} \mid j \in J^u_k)$.	
\end{claim}

\begin{claim} [Modification of Claim~\ref{claim:u-2}]
    \label{claim:u-2-m}
For each epoch $k$ and any $\eta \in (0, 1)$, we have,
    \begin{equation}%
         \opt(l_k y_j^{(k)} \mid j \in J^u_k) \leq \frac{1}{\eta} \opt(l_k y_j^{(k)} \mid j \in X^u_k) + \frac{1}{1-\eta} \opt(l_k y_j^{(k)} \mid j \in Y^u_k)  \ .
    \end{equation}
\end{claim}

\begin{lemma}[Modification of Lemma~\ref{lem:epoch-bound}]\label{lem:epoch-bound-m}
	For each epoch $k$, we have $n^u_k l_k \leq \left( \frac{2\eta}{(\eta - 2\beta(1+2\eps)) (1-\eta)} \right) \opt(l_k y_j^{(k)} \mid j \in Y^u_k) $.
\end{lemma}

\begin{proof}    
    We first note that $n_k \leq n^u_k / (1 - \eps)$: After the transition, all jobs are effectively considered, and therefore, this holds true obviously. Before the transition, we have $n_k \geq g / \eps$. Since $n_k \le n^u_k + g \leq n^u_k +\eps n_k$, we have the claimed inequality. 

    Combining Claims~\ref{claim:u-1-m} and \ref{claim:u-2-m} with the fact that $|X_k| \leq \beta n_k \leq \beta n^u_k / (1 - \eps) \leq \beta(1+2\eps) n^u_k$, we obtain
    \begin{align*}
	n^u_k l_k &\leq 2 \left( \frac{1}{\eta} \opt(l_k y_j^{(k)} \mid j \in X^u_k) + \frac{1}{1-\eta} \opt(l_k y_j^{(k)} \mid j \in Y^u_k) \right) \\
	&\leq 2 \left( \frac{1}{\eta} l_k |X^u_k| + \frac{1}{1-\eta} \opt(l_k y_j^{(k)} \mid j \in Y^u_k) \right) \\
	&= 2 \left(  \frac{\beta(1+2\eps)}{\eta} l_k n^u_k + \frac{1}{1-\eta} \opt(l_k y_j^{(k)} \mid j \in Y^u_k) \right) \ .
    \end{align*}
    By rearranging terms, we obtain the desired inequality.
\end{proof}

\begin{lemma}[Modification of Lemma~\ref{lem:sum-up-pj}]
\label{lem:sum-up-pj-m}
 For every job $j$, it holds that $\sum_{k: j \in Y^u_k} l_k y_j^{(k)} \leq p_j$.
\end{lemma}

\begin{corollary}[Modification of Corollary~\ref{cor:u-final}]
    \label{cor:u-final-m}
    $\sum_{k=1}^K \opt(l_k y_j^{(k)} \mid j \in Y_k) \leq 1.81 \cdot \opt(p)$.
\end{corollary}

Finally, we upper bound the total completion time objective of this two-stage modification of SNAP. Specifically, we have, 
	\begin{align*}
	       \alg(p) &\leq 4 (1+\eps) \sum_{k=1}^K n_k l_k \\
           &\leq  4 ( 1+\eps) \left( \frac{2\eta}{(\eta - 2\beta(1+2\eps)) (1-\eta)} \right) \sum_{k=1}^K \opt(l_k y_j^{(k)} \mid j \in Y_k) \\
           &\leq 14.48  (1+\eps)  \left( \frac{\eta}{(\eta - 2\beta(1+2\eps)) (1-\eta)} \right) \opt(p)  \\
            & \leq 14.48 (1+\varepsilon)^2 \opt(p) \ ,
    \end{align*}
    where the first inequality is by \Cref{lem:alg-virt-bound-m},
    the second inequality is by \Cref{lem:epoch-bound-m},
    and the third inequality is by \Cref{cor:u-final-m} and by choosing $\eta = \frac{\varepsilon}{2(1+\varepsilon)}$ and $\beta = \frac{\varepsilon^2}{4(1+\varepsilon)(2+\varepsilon)(1 + 2\eps)}$ for any sufficiently small $\varepsilon > 0$.
By scaling $\varepsilon$ appropriately, we essentially maintain the \emph{same} competitive ratio as that stated in Theorem~\ref{thm:u-cr}.

\paragraph{Bounding number of preemptions.}
We now turn our attention to upper-bounding the number of preemptions. By following exactly the same steps as the proof of Theorem~\ref{thm:u-preemption-bound}, we can bound the number of preemptions until epoch $K'$ by:
\[
\sum_{k \leq K'} n_k \leq O \bigg( \frac{1}{\beta \delta} \sum_j  \log \bigg(\frac{p_j}{\hat{p}_j} + 1 \bigg) \bigg).
\]

Next, consider the jobs alive at the transition time $e_{K'+1}$. By the definition of the transition, we know $|J_{K'+1}| \leq g \cdot (1 + 1/\varepsilon)$. Since we reset the predicted sizes of all jobs in $J_{K'+1}$ such that they are all henceforth underestimated, we can apply Theorem~\ref{thm:u-preemption-bound}. Letting $p_{\max} := \max_j p_j$, the number of preemptions for these jobs is at most:
\[
O \bigg( \frac{1}{\beta \delta} \sum_{j \in J_{K'+1}} (1+\log p_j) \bigg) 
= O \bigg( \frac{g}{\beta \delta \varepsilon} \log p_{\max} \bigg)
= O \bigg( \frac{g}{\varepsilon^4} \log p_{\max} \bigg) \ .
\]

To summarize, we have shown the following theorem:

\begin{theorem} \label{thm:u-small-over}
For any sufficiently small $\varepsilon > 0$, assuming there are at most $g$ overpredicted jobs, there exists an algorithm that is $14.48(1+\varepsilon)$-competitive for minimizing total completion time while performing at most
$
O( \frac{1}{\varepsilon^3} ( \sum_j  \log ( \frac{p_j}{\hat{p}_j} + 1 ) + \frac{g}{\varepsilon} \log p_{\max} ) )
$
preemptions and migrations.
\end{theorem}

\section{A general framework for PSP}\label{app:framework}

In this section, we present the full description of our framework for PSP, and the proof of \Cref{thm:psp-main}.

Our framework works in epochs $1,2,\ldots,K$. %
Epoch $k$ starts at time $e_k$ and ends at time $e_{k+1}$; Epoch $1$ starts at time $e_1=0$. %
Let $J_k$ denote the set of available jobs 
at the start of epoch $k$ and let $|J_k| = n_k$.
At time $e_k$, we plan a complete schedule
for epoch $k$ as described in the following:
\begin{enumerate}
\item \textbf{Compute PF Rates.} 
We compute the PF rates $y^{(k)}$ for the jobs $J_k$. Formally, 
$y^{(k)}$ is the solution to %
\begin{equation}
  \renewcommand{\arraystretch}{1.4}
  \begin{array}{rr>{\displaystyle}rcl>{\quad}l}
   (\CP(J))& \operatorname{max} &\multicolumn{3}{l}{\displaystyle\sum_{j \in J} w_j \cdot \log(y_j)} \\
    &\text{s.t.} &\sum_{j \in J} b_{dj} \cdot y_j &\le& 1 &\forall d \in [D] \\
    &&y_j &\ge& 0 &\forall j \in J
  \end{array}
  \notag
\end{equation}
\item \textbf{Compute Next Checkpoints.}
For each job $j \in J_k$, we define the remaining processing requirement to
reach its next checkpoint as $u_{j,k} = (1+\delta)^{h+1} - q_j(e_k)$, where $h \in \mathbb{N}$ s.t.\ $(1+\delta)^h \le \max\{ q_j(e_k), \hat p_j \} < (1+\delta)^{h+1}$.
\item \textbf{Determine the Simulation Duration.}
Let $l_k$ denote the length of the shortest time interval such that, if jobs are processed at rates $y^{(k)}$ during the time interval, at least $\lceil \beta n_k \rceil$ jobs from $J_k$ reach their next checkpoint. 
The targeted processing requirement for each job $j \in J_k$  is then $v_{j,k} = \min\{u_{j,k}, l_k y^{(k)}_j\}$.
\item \textbf{Convert to a $O(1)$-Preemptive Schedule.}
We now compute an $\alpha$-approximate $O(1)$-preemptive schedule (cf.\ \Cref{def:approximate-preemptive}) for jobs $J_k$ with processing requirements $v_{j,k}$
for each $j \in J_k$. If we start this schedule at time $e_k$, it will complete by time at most $e_{k+1} \leq e_k + \alpha l_k$. Then, Epoch $k$ ends and we start Epoch $k+1$ if not all jobs have been completed yet.
\end{enumerate}

\begin{theorem}\label{thm:psp-preemptions}
	The total number of preemptions performed by our framework is at most
	$O(\frac{1}{\beta \delta} \sum_j (1+\log_2 p_j / \hp_j))$.
\end{theorem}

\begin{proof}	
In each epoch $k$, the $\alpha$-approximate $O(1)$-preemptive schedule uses by definition at most $O(1)$ preemptions per job, so $O(n_k)$ preemptions in total.
By our choice of $l_k$, at least $\beta n_k$ jobs each reach a checkpoint during epoch $k$.
Since there can be at most $\sum_j (1 + \lceil \log_{1+\delta} p_j / \hp_j \rceil)$ checkpoints in total, 
there can be at most $\beta \cdot \sum_j (1 + \lceil \log_{1+\delta} p_j / \hp_j \rceil)$ many epochs.
Thus, the total number of preemptions used by our algorithm is at most
$O(\frac{1}{\beta} \sum_j \log_{1+\delta} p_j / \hp_j) = O(\frac{1}{\beta \delta}  \sum_j (1+\log_2 p_j / \hp_j))$.
\end{proof}

\begin{lemma}\label{lemma:psp-virtual}
For any $\beta \in (0,1)$, the virtual algorithm achieves a total weighted completion time at most $\frac{27(1+\delta)}{(1 - \beta)^3}$ times the optimum for PSP.
\end{lemma}

\begin{proof}
Consider any epoch $k$ in the schedule $S$ of the virtual algorithm for this epoch.
For simplicity, we assume without loss of generality 
that epoch $k$ starts at time $0$ and ends at time $l_k$.
Let $X_k$ denote the set of jobs that reach their next checkpoint during epoch $k$.
For each $X_k$, let $l_{j,k}$ denote the time at which job $j$ reaches its next checkpoint, that is, a processing requirement of $v_{j,k}$. 

We next construct an instance $I'_k$ for $\beta$-robust PSP.
PF's initial rates $y$ for $I'_k$ are equal to $y^{(k)}$, as both are based on the set of jobs $J_k$ and the rates are unique.
For any time $t \in [0, l_k]$, 
we choose the nullified set of jobs to be $N(t) = \{j \in X_k \mid t \leq l_k - l_{j,k} \}$.
By the definition of $N(t)$, for each job $j \in J_k \setminus X_k$, we have that $j \notin N(t)$ for all $t \in [0, l_k]$.
For each job $j \in X_k$, we have that $j \in N(t)$ for all $t \in [0, l_{j,k}]$.
This implies that in the schedule $S'$ of PF for $I'_k$,
every job completes exactly at time $l_k$.
Hence $|N(t)| \leq \beta |J_k| = \beta |J(t)|$ for all $t \in [0, l_k]$, so
$I'_k$ is a $\beta$-robust PSP instance.
In conclusion,
the total weighted completion time of PF for $I'_k$, denoted by $\pf(I'_k)$, is at most $n_k l_k$.

Similarly to \Cref{lem:u-alg-virt-bound},
we have
\[
    \alg^{virt} \leq \sum_{k \geq 1} n_k l_k = \sum_{k \geq 1} \pf(I'_k) \ .
\]
Now we can concatenate all instances $I'_k$ to form a new instance $I'$.
It is easy to check that this is again a $\beta$-robust PSP instance where 
each job $j$ has a total processing requirement equal to $p_j$.
Moreover, the total weighted completion time of PF for $I'$ is equal to $\sum_{k \geq 0} \pf(I'_k)$.
Therefore,
\[
    \alg^{virt} \leq \sum_{k \geq 1} n_k l_k = \sum_{k \geq 1} \pf(I'_k) = \pf(I') \ .
\]
Finally, by \Cref{thm:rpf}, we have $\pf(I') \leq 27/(1-\beta)^3 \cdot \opt$.
\end{proof}

Since the total completion time of the framework is at most $\alpha$ times the total completion time of the virtual algorithm, \Cref{thm:psp-preemptions} and \Cref{lemma:psp-virtual} imply \Cref{thm:psp-main}.

\section{Application to malleable scheduling}\label{app:malleable}

We apply our framework to malleable job scheduling
on identical and unrelated  machines. 
We are given a set of machines $M$ and for every job $j \in J$ a speed-up function $f_j \colon \mathbb{R}_{\geq 0} \to \mathbb{R}_{\geq 0}$.
In a job-to-machine assignment, let $x_{ij}(t)$ be the fraction of machine $i$ given to job $j$ at time $t$.
The processing rate of job $j$ at time $t$ is given by $f_j(\sum_{i} s_{ij} x_{ij}(t))$.
We assume that each $f_j$ is non-decreasing and concave.
This is a PSP with the rate polytope 
\(
    \mathcal P = \{f(x) \mid \sum_{i \in M} x_{ij} \leq 1 \, \forall j \in J, \sum_{j \in J} x_{ij} \leq 1 \forall i \in M \}
\);
see \cite{ImKM18} for details.

We present algorithms
for computing $\alpha$-approximate $O(1)$-preemptive schedules (cf.~\Cref{def:approximate-preemptive}), %
which together with \Cref{thm:psp-main} 
imply competitive algorithms with few preemptions. 

We start with the case of $m$ identical parallel machines.
Here, we additionally assume that each function $f_j$ is piecewise linear with breakpoints at integers points. 

\begin{restatable}{theorem}{thmMalleableIdentical}
    \label{thm:m-i}
For any $\varepsilon > 0$, there is a $54(1+\varepsilon)$-competitive algorithm that incurs at most $O(\frac{1}{\varepsilon^2} \sum_j (1+ \log_2 p_j / \hp_j))$ preemptions
for minimizing the total completion time for malleable job scheduling on identical parallel machines if each function $f_j$ is piecewise linear with breakpoints at integer points.
\end{restatable}

\begin{proof}
We provide an algorithm for computing $2$-approximate $0$-preemptive schedules
for the problem. The theorem then follows from \Cref{thm:psp-main}.

Consider fixed rates $y$ that correspond to a fractional schedule.
where each job $j$ is assigned to $x_j$ machines and a length $L$.
Let $J_1 = \{j \mid x_j \geq 1 \}$ and $J_2 = \{j \mid x_j < 1 \}$.
We create a non-preemptive integral schedule as follows:
For each job $j \in J_1$, we assign job $j$ to $\lfloor x_j \rfloor$ machines $M_j$
and those machines only work on job $j$ throughout the whole schedule. 
Since $f_j$ is concave and piecewise linear with breakpoints at integers, and since $x_j \geq 1$,
we have $f_j(x_j) \leq 2 f_j(\lfloor x_j \rfloor)$.
Thus, the completion time of $j$ in the integral schedule is at most $2L$.
Note that jobs in $J_1$ use at most $m_1 = \lfloor \sum_{j \in J_1}  x_j  \rfloor$ many machines.
Next, consider all jobs $j \in J_2$.
We greedily schedule those jobs on the remaining $m - m_1 = \lceil \sum_{j \in J_2}  x_j  \rceil$ machines via list scheduling, where each job gets one machine until it finishes.
By the well-known list scheduling properties, this ensures that all jobs in $J_2$ complete within at most $2L$ time.
Since no job is preempted, this completes the proof of the theorem.
\end{proof}

We next present our results for unrelated machines. For this problem we additionally assume that $f_j(s) \cdot s$ is non-increasing as a function of $s$. This is a common \emph{non-decreasing work} assumption, see e.g.,~\cite{FotakisMP23}.

\thmMalleableUnrelated*

\begin{proof}
We provide an algorithm for computing $e/(e-1)$-approximate $0$-preemptive schedules
for the problem. The theorem then follows from \Cref{thm:psp-main}.

Consider fixed rates $y$ that correspond to a fractional schedule,
where each job $j$ is fractionally assigned to machine $i$ with $x_{ij}$.
Assume w.l.o.g.\ by scaling that the length of the fractional schedule is $1$.
Define 
\begin{itemize}
\item $\sigma_j(x) := \sum_{i} \lambda_{ij} x_{ij}$ for every job $j$,
\item $\gamma_j := \min \{q \in \mathbb R \mid y_j \leq f_j(q) \}$ for every job $j$, and
\item $r_{ij} := \max\{\lambda_{ij}, \gamma_j \}$ for every machine $i$ and job $j$.
\end{itemize}
Consider the new allocation
\[
    \hx_{ij} = \frac{\lambda_{ij} \cdot x_{ij}}{\sigma_j(x)} 
\]
for all machines $i$ and jobs $j$.
We will show that the allocation $\hx$ satisfies the following two constraints:
\begin{alignat*}{3}
    \sum_{i} \hx_{ij} & = 1 &\quad& \forall j  \\
    \sum_{j} \frac{y_j}{f_j(r_{ij})} \cdot \frac{r_{ij}}{\lambda_{ij}} \cdot \hx_{ij} & \leq 1 &\quad& \forall i 
\end{alignat*}
Fotakis, Matuschke, and Papadigenopoulos~\cite{FotakisMP23} showed that if these two constraints hold, then there exists a non-preemptive schedule with a makespan of at most $e/(e-1)$, which implies our theorem.

The first constraint holds immediately from the definition of $\hx$.
To verify the second constraint for any machine $i$, consider a job $j$ with $x_{ij} > 0$.
Note that $\sigma_j(x) \geq x_{ij} s_{ij}$ and $\sigma_j(x) \geq \gamma_j$, because $f_j(\sigma_j(x)) = y_j$. Thus, we have $\sigma_j(x) \geq x_{ij} r_{ij}$.
Then, it holds that
\begin{align}
	\frac{y_j}{f_j(r_{ij})} \cdot \frac{r_{ij}}{s_{ij}} \cdot \hx_{ij}
	= \frac{y_j}{f_j(r_{ij})} \cdot \frac{r_{ij}}{s_{ij}} \cdot \frac{s_{ij} \cdot x_{ij}}{\sigma_j(x)} 
	&=\frac{y_j}{f_j(r_{ij})} \cdot \frac{r_{ij} \cdot x_{ij}}{\sigma_j(x)} \notag \\
	&\leq \frac{y_j}{f_j(\sigma_j(x) / x_{ij})} \label{eq:speedup1}  \\
	&\leq x_{ij}\frac{y_j}{f_j(\sigma_j(x))} \label{eq:speedup2} = x_{ij}  \ ,
\end{align}
where Inequality~\eqref{eq:speedup1} follows from $\sigma_j(x) \geq x_{ij} r_{ij}$ and Fact~5 in \cite{FotakisMP23},
and Inequality~\eqref{eq:speedup2} follows from the concavity of $f_j$.
Summing over all jobs $j$ gives the second constraint, because $\sum_{j} x_{ij} \leq 1$.
\end{proof} 

\section{Proportional Fairness for $\beta$-robust PSP}\label{sec:robustpf}

In this section, we analyze PF for $\beta$-robust PSP and show \Cref{thm:rpf}, which we first restate.
Our analysis holds even for the setting where jobs arrive online over time at their release dates $r_j$.

\thmRobustPf*

We first restate the PF algorithm. At every time $t$, it computes the processing allocation $y(t)$ that
is optimal for the following convex program $\CP(J(t))$ for the active jobs $J(t)$ at time $t$.
\begin{equation}
  \renewcommand{\arraystretch}{1.5}
  \begin{array}{rr>{\displaystyle}rcl>{\quad}l}
   (\CP(J))& \operatorname{max} &\multicolumn{3}{l}{\displaystyle\sum_{j \in J} w_j \cdot \log(y_j)} \\
    &\text{s.t.} &\sum_{j \in J} b_{dj} \cdot y_j &\le& 1 &\forall d \in [D] \\
    &&y_j &\ge& 0 &\forall j \in J
  \end{array}
  \notag
\end{equation}

In the $\beta$-robust PSP setting, the adversary can then select a set of jobs $N(t) \subseteq J(t)$ such that $|N(t)| \leq \beta \cdot |J(t)|$, for some fixed $\beta \in (0,1)$. The final processing allocation of the algorithm is then 
\[        \hy_j(t) := \begin{cases}
            0 &\text{if } j \in N(t) \\
            y_j(t) &\text{else.} 
            \end{cases} 
\]

Note that if $y(t) \in \calP$ then also $\hy(t) \in \calP$ since we assume that $\calP$ is downward closed.

Our analysis is based on the analysis of PF for PSP~\cite{ImKM18,JLM25}.
Our main contribution here is to carefully integrate that changed rates of nullified
jobs into the traditional analysis.
A critical issue is that the new rate vector $\hy$ is not necessarily an
optimal solution to $\CP(J)$ and thus does not satisfy 
crucial optimality conditions on which the proof in \cite{ImKM18,JLM25} relies on.

\subsection{KKT conditions}

To characterize the optimal solution of $(\CP(J))$, we consider its KKT conditions for the Lagrange multipliers $\eta_d$ for each $d \in [D]$~\cite{BV2014}.
\begin{align}
 \frac{1}{y_j} - \sum_{d=1}^D b_{dj} \eta_{d} &= 0 \quad \forall j \in J \label{psp-kkt1}\\
 \eta_{d} \cdot \biggl(1- \sum_{j \in J} b_{dj} y_{j} \biggr) &= 0 \quad \forall d \in [D]  \label{psp-kkt2} \\
 \eta_{d} & \geq 0 \quad \forall d \in [D] \label{psp-kkt3}
\end{align}
A direct consequence of the KKT conditions is the following lemma, which is a classic connection between the number of active jobs and their dual multipliers~\cite{ImKM18,eisenberg1959consensus}.
We refer to \cite{ImKM18,JLM25} for a proof.
\begin{lemma}[\cite{ImKM18}]\label{lemma:kkt-multiplier-bound}
 Let $J$ be a set of jobs, and let $\eta$ be optimal 
 Lagrange multipliers for $(\CP(J))$. Then $\sum_{d=1}^D \eta_{d} = |J|$.
\end{lemma}

\subsection{Linear programming relaxation}
We compare the performance of PF to an optimal schedule that is allowed to run at a slower speed.
More precisely, we compare PF to an optimal solution that runs at speed~$1/\kappa$ for some parameter $\kappa \geq 1$. The following linear program is a relaxation of this optimal solution, where the variables~$y_{jt}$ indicate the total amount of processing that job~$j$ receives in the interval~$[t,t+1]$. 
\begin{equation}
 \renewcommand{\arraystretch}{1.5}
 \begin{array}{rr>{\displaystyle}rc>{\displaystyle}l>{\quad}l}
  (\LP(\kappa))& \operatorname{min} &\multicolumn{3}{l}{\displaystyle\sum_{j \in J}  \sum_{t \geq r_j} \frac{y_{jt}}{p_j}  \Bigl(t + \frac{1}{2} \Bigr)}  \\
  &\text{s.t.} &\sum_{t \geq r_j} y_{jt} &\ge& p_j & \forall j \in J  \\
  && \sum_{j \in J} b_{dj} \cdot y_{jt} &\le& \frac{1}{\kappa} &\forall d \in [D],\ \forall t \geq 0 \\
  && y_{jt}  &\ge& 0 &\forall  j \in J,\ \forall t \geq r_j
 \end{array}
 \notag
\end{equation}
Its dual can be written as follows.
\begin{equation}
 \renewcommand{\arraystretch}{1.9}
 \begin{array}{rr>{\displaystyle}rc>{\displaystyle}l>{\;}l}
  (\DP(\kappa))& \mathrm{max} &\multicolumn{3}{l}{\displaystyle\sum_{j \in J} a_j - \sum_{d = 1}^D \sum_{t \geq 0} b_{dt}} \\
  &\text{s.t.} &\frac{a_j}{p_j} - \frac{1}{p_j} \Bigl(t + \frac{1}{2} \Bigr) &\le& \kappa \sum_{d=1}^D b_{dj}  b_{dt} &\forall j \in J,\ \forall t \geq r_j \\
  &&a_j, b_{dt} &\ge& 0 &\forall j \in J,\ \forall t \geq 0,\ \forall d \in [D]
 \end{array}
 \notag
\end{equation}

\subsection{Dual fitting}

Fix an instance and consider the schedule produced by PF for that instance. Let $C_j$ denote the completion time of job $j$ in that schedule, and let $\alg:=\sum_{j=1}^n C_j$ denote the
objective function value of PF. For every time $t$, let $U(t)$ denote the set of unfinished (and potentially unreleased) jobs at time $t$. %
Let $\lambda$ be a constant with $\beta < \lambda < 1$ that we fix later.
In the following, we assume w.l.o.g.\ by scaling that all weights are integers.

For every time~$t$, let $\zeta(t)$ be the  $\lambda$-quantile of the values~$\hy_j(t)/p_j$, $j \in U(t)$. More formally, if $Z_t$ denotes the sorted (ascending) list of length~$|U(t)|$ composed of~$\hy_{j}(t)/p_j$ for every~$j \in U(t)$, then $\zeta(t)$ is the $\lceil\lambda |U(t)| \rceil$th value in~$Z_t$.
Further, for $t \ge 0$ let $\eta_d(t)$ be the optimal Lagrange multiplier corresponding to the constraint~$d \in [D]$ of $(\CP(J(t)))$. We consider the following assignment of dual variables for $(\DP(\kappa))$:
\begin{itemize}
    \item $a_j := \sum_{t=0}^{C_j} a_{jt}$ for every job $j \in J$, where $a_{jt} :=  \bone\bigl[\hy_j(t)/p_j \leq \zeta(t)\bigr] \cdot \bone \bigl[ j \in U(t) \setminus N(t) \bigr]$ for $t \ge 0$ and $j \in U(t)$,
    \item $b_{dt} := \frac{1}{\kappa} \sum_{t' \geq t} \zeta(t') \cdot \eta_{d}(t')$ for every $d \in [D]$ and time $t \geq 0$.
\end{itemize}

We first show in the following lemma that the objective value of $(\DP(\kappa))$ for this assignment upper bounds a constant fraction of the algorithm's objective value.

\begin{lemma}\label{lemma:pf-dual-objective}
    It holds that $\bigl(\lambda - \beta - \frac{1}{(1-\lambda)\kappa}\bigr) \alg \leq \sum_{j \in J} a_j - \sum_{d=1}^D \sum_{t \geq 0} b_{dt}$.
\end{lemma}

The lemma follows from the following two statements.

\begin{lemma}\label{lemma:pf-dual-obj1}
    It holds that $\sum_{j \in J} a_j \geq (\lambda - \beta) \cdot \alg$.
\end{lemma}

\begin{proof}
  Consider a time~$t$. Observe that~$\sum_{j \in U(t)} a_{jt}$ contains the total number of jobs~$j$ that satisfy~$\hy_j(t)/p_j \leq \zeta(t)$ and $j \in U(t) \setminus N(t)$. 
  By the definition of $\zeta(t)$, we conclude that $\hy_j(t)/p_j \leq \zeta(t)$ satisfy at least~$\lambda \cdot |U(t)|$ jobs.
  By our assumption on $N(t)$, we conclude that $j \in U(t) \setminus N(t)$ satisfy at least~$(1 - \beta) |U(t)|$ jobs.
  Thus, at time $t$, we have a total contribution of at least $(1-(1-\lambda)-\beta) \cdot |U(t)| = (\lambda - \beta) \cdot |U(t)|$ to $\sum_{j \in J} a_j$.
  The statement then follows because $\alg = \sum_{t \geq 0} |U(t)|$.
\end{proof}

The following lemma and its proof appear also in \cite{JLM25}. 
We include it here for completeness.
\begin{lemma}[\cite{JLM25}]\label{lemma:pf-dual-obj2}
    At any time $t$, it holds that $\sum_{d=1}^D b_{dt} \leq \frac{1}{(1-\lambda) \kappa} |U(t)|$.
\end{lemma}
\begin{proof}
  Fix a time $t$. Observe that for every $t' \geq t$, the definition of $\zeta(t')$ implies that $|U(t')| \cdot \bone \bigl[ \hy_j(t')/p_j \geq \zeta(t') \bigr]\ge (1 - \lambda) \cdot |U(t')|$. Thus,
  \begin{equation}
    \zeta(t') \cdot (1 - \lambda) \cdot |U(t')| \leq  |U(t')| \cdot \zeta(t') \cdot \bone \Big[ \frac{\hy_j(t')}{p_j} \geq \zeta(t') \Big] \leq |U(t')| \cdot \frac{\hy_j(t')}{p_j} \ . \label{eq:pf-dual-obj2-eq1}
  \end{equation}

  The definition of $b_{it}$ and \Cref{lemma:kkt-multiplier-bound} imply
  \begin{align*}
    \sum_{d=1}^D b_{dt}  = \sum_{d=1}^D \frac{1}{\kappa}\sum_{t' \geq t}   \zeta(t')  \cdot \eta_d(t')
    =  \frac{1}{\kappa} \sum_{t' \geq t} \zeta(t') \sum_{d=1}^D \eta_d(t') &= \frac 1 \kappa \sum_{t' \ge t} \zeta(t') \cdot |J(t')| \\
    &\leq \frac{1}{\kappa} \sum_{t' \geq t} \zeta(t') \cdot |U(t')| \ ,
  \end{align*}
  where the final inequality holds because $J(t') \subseteq U(t')$ for all times $t'$.
 Using~\eqref{eq:pf-dual-obj2-eq1}, we conclude for the right-hand side of the above inequality that
  \begin{align*}
    \frac{1}{\kappa} \sum_{t' \geq t} \zeta(t') \cdot |U(t')|
    &\leq  \frac{1}{(1 - \lambda)\kappa} \sum_{t' \geq t} \sum_{j \in U(t')} \frac{\hy_j(t')}{p_j} \\
    &\leq \frac{1}{(1 - \lambda) \kappa} \sum_{j \in U(t)} \sum_{t' \geq t}  \frac{\hy_j(t')}{p_j}
    \leq \frac{1}{(1 - \lambda) \kappa} \sum_{j \in U(t)} 1 \ .
  \end{align*}
  The second inequality follows from $U(t') \subseteq U(t)$ for $t' \geq t$. The third inequality holds because $\sum_{t' \geq t} \hy_j(t') \leq p_j$ for every job $j$. This concludes the proof of the lemma.
\end{proof}

Next, we show dual feasibility.

\begin{lemma}\label{lemma:pf-dual-feasibility}
 The dual solution $(a_j)_j$ and $(b_{it})_{i,t}$ is feasible for $(\DP(\kappa))$.
\end{lemma}

\begin{proof}
  First, note that the dual solution is non-negative.
  To verify the dual constraint, fix a job~$j$ and a time $t \geq r_j$. The left-hand side of the dual constraint is then given by
    \begin{align*}
        \frac{ a_j }{p_j} - \frac{1}{p_j} \Bigl(t + \frac 1 2\Bigr)
        &\leq \sum_{t' = 0}^{C_j} \frac{1}{p_j} \cdot \bone\Big[\frac{\hy_j(t')}{p_j} \le \zeta(t')\Big] \cdot \bone\Big[ j \in U(t') \setminus N(t') \Big] - \frac{1}{p_j} t \\
        &\le \sum_{t' = t}^{C_j} \frac{1}{p_j} \cdot \bone\Big[ \frac{\hy_j(t')}{p_j} \leq \zeta(t') \Big] \cdot \bone\Big[ j \in U(t') \setminus N(t') \Big] \\
        &=  \sum_{t' = t}^{C_j} \frac{w_j}{y_j(t')} \cdot \frac{\hy_j(t')}{p_j} \cdot \bone\Big[ \frac{\hy_j(t')}{p_j} \leq \zeta(t') \Big] \cdot \bone\Big[ j \in U(t') \setminus N(t') \Big] \ .
      \end{align*}
      The last equality is possible because for every $r_j \leq t' \leq C_j$, PF assigns job $j \in U(t)$ a positive rate $y_j(t') > 0$ and if $j \in U(t') \setminus N(t')$, then $y_j(t') = \hy_j(t')$.
      By the KKT condition~\eqref{psp-kkt1}, we can substitute $\frac{1}{y_j(t')} = \sum_{d=1}^D b_{dj} \cdot \eta_d(t')$ above. Since for every $d \in [D]$ we have that
      \begin{align*}
        \sum_{t' = t}^{C_j} b_{dj} \cdot \eta_d(t') \cdot \frac{\hy_j(t')}{p_j} \cdot \bone\Big[ \frac{\hy_j(t')}{p_j}
       \leq \zeta(t') \Big]
       &\leq \sum_{t' = t}^{C_j}  b_{dj} \cdot \eta_d(t') \cdot \zeta(t') \\
       & \leq  b_{dj} \sum_{t' \geq t} \eta_d(t') \cdot \zeta(t')
        = \kappa  b_{dj} b_{dt} \ ,
    \end{align*}
    we obtain the right-hand side of the dual constraint by summing over all $d \in [D]$.
\end{proof}

We can finally prove \Cref{thm:rpf}.

\begin{proof}[Proof of \Cref{thm:rpf}]
  Weak duality, \Cref{lemma:pf-dual-objective}, and~\Cref{lemma:pf-dual-feasibility} imply
  \begin{align*}
    \kappa \cdot \opt
    \geq  \sum_{j \in J} a_{j} - \sum_{d=1}^D \sum_{t \geq 0} b_{dt} \geq \Bigl(\lambda - \beta -  \frac{1}{(1-\lambda)\kappa}\Bigr) \cdot \alg 
  \end{align*}
  for all parameters that satisfy $\lambda - \beta - \frac{1}{(1-\lambda)\kappa} > 0$.

  We now set $\lambda := \frac{\beta+2}{3}$ and $\kappa := \frac{9}{(1-\beta)^2}$, which satisfy the above condition.
  This gives 
  \[
  \lambda - \beta - \frac{1}{(1-\lambda)\kappa} 
  =  \frac{2(1-\beta)}{3} - \frac{1-\beta}{3} = \frac{1-\beta}{3} \ ,
  \]
and thus, we conclude that $\alg \leq \frac{27}{(1-\beta)^3} \cdot \opt$.
\end{proof}

\section{Other related work}
    \label{app:related}

Since various objectives are considered in the literature, we mostly focus on the total (weighted) completion time objective. Unless otherwise specified, we assume that preemption is permitted in the results discussed below. For brevity, ``completion time'' will hereafter refer to the objective of minimizing the sum of completion times.

\textbf{Single Machine:}  %
In the clairvoyant scenario, it is folklore that Shortest Remaining Time First (SRPT) is optimal. Note that SRPT becomes Shortest Job First (SJF), which processes jobs in non-decreasing order of their processing time. For non-clairvoyant scheduling, Round-robin (RR) is known to be $2$-competitive for  completion time, which is the best competitive ratio achievable in this setting \cite{MPT94}. Shortest Elapsed Time First (SETF) processes all jobs that have received the least amount of processing equally; note that SETF coincides with RR if all jobs arrive simultaneously. Recently, it was shown that SETF is $2$-competitive for (weighted) completion time~\cite{JagerSWW23}.%

While being $O(1)$-competitive, both SETF and RR incur excessively many preemptions. Multilevel (Adaptive) Feedback (MLF) seeks to minimize the number of preemptions while retaining $O(1)$-competitiveness for completion time. MLF can be parameterized to achieve a $(2+\epsilon)$-competitive ratio for any $\epsilon > 0$ \cite{MPT94}.

\textbf{Identical Machines:} The performance guarantees of SRPT and RR for a single machine extend to parallel identical machines \cite{MPT94, BBEM12}; see also \cite{DengGBL00, MoseleyV22} for closely related work. Under the resource augmentation framework, SETF is shown to be $O(1)$-competitive for identical machines \cite{KalyanasundaramP00}. 

\textbf{Completion Time vs. Flow Time:} The optimal schedule for total completion time is identical to  that of the total flow time objective, where a job's flow time is defined as $F_j = C_j - r_j$. However, from an approximation and competitive analysis standpoint, the flow time objective is significantly more challenging. Consequently, online algorithms are often provided with extra speed and compared against the offline optimum. It is well-established that an algorithm that is $O(1)$-competitive with $O(1)$-speed for flow time is also $O(1)$-competitive for completion time. For example, while \cite{KalyanasundaramP00} primarily studies flow time minimization, their results readily translate to the completion time objective.

\textbf{Unrelated Machines:} In the offline setting, minimizing total completion time on unrelated machines is solvable in polynomial time if all jobs arrive at $t=0$. If jobs have release times, the problem becomes NP-hard and, in fact, APX-hard \cite{hoogeveen1998non}. Constant-factor approximations are well-established \cite{schulz2002scheduling, skutella2001convex, sethuraman1999optimal, BSS21, im2023improved, harris2024dependent}, and the current state-of-the-art approximation ratio is $1.36 + \varepsilon$ due to  \cite{li2025approximating}. In the clairvoyant online setting, utilizing immediate-dispatch and non-migratory algorithms while assuming each machine runs SPT or SRPT, assigning an arriving job to the machine with the least marginal increase in the objective is $O(1)$-competitive \cite{AnandGK12,GuptaMUX20}. For non-clairvoyant algorithms, SelfishMigrate and Proportional Fairness (PF) were analyzed with competitive ratios of $32$ and $4.62$, respectively \cite{ImKMP14, JLM25}. Various clairvoyant algorithms were also explored in \cite{HSSW97, ChakrabartiPSSSW96, LindermayrMR23}.

\textbf{Polytope Scheduling:} Polytope scheduling is a general framework that captures various scheduling objectives and encompasses all the aforementioned problems. For this general scheduling problem, \cite{ImKM18} showed that Proportional Fairness (PF) is $128$-competitive for total weighted completion time, a result that was later improved to $27$ by \cite{JLM25}. Furthermore, \cite{JLM25} demonstrated that PF achieves surprisingly strong competitive ratios for various special cases, including the unrelated machines setting.%

\textbf{Online Scheduling with Predictions:} As alluded to previously, scheduling has become a primary application for ML-augmented algorithms. Various objectives were considered. For example, total completion time \cite{purohit2018improving, ImKQP23,LindermayrM25,DinitzILMV22}, total flow time \cite{azar2021flow,mitzenmacher2019scheduling}, makespan minimization \cite{lattanzi2020online,li2021online,BalkanskiOSW25}, and minimizing energy consumption \cite{bamas2020learning,AntoniadisGS22}.

\section{Further experimental details and results}
    \label{app:exp}

\subsection{Algorithm details}
\label{app:alg-detail}

We provide further details of the algorithms. 

\begin{enumerate}
\item \textbf{SNAP} (coupled with PMLF): Recall that the SNAP algorithm aims to process each job $j$ by $v_{j,k}$ units in epoch $k$. While a machine $i$ could technically idle after meeting the specific processing targets $v_{j, k}$ for its assigned jobs, such idling results in an unnecessary loss of efficiency. 

Therefore, we introduce an engineering adaptation: while jobs are assigned to machines following Step 4, we execute the \textbf{PMLF} (Preemptive Multi-Level Feedback) policy on each individual machine to process its assigned jobs. An epoch concludes once a $\beta$ fraction of jobs are exhausted.

This approach serves as a practical adaptation that maintains the underlying spirit and analysis of SNAP. Each epoch is designed to exhaust a constant fraction of jobs, mimicking the schedule of \textbf{PF} (Proportional Fairness). This behavior is indirectly achieved through PMLF; we know Step 4 processes jobs without exhausting more than $\beta n_k$ jobs within $4l_k$ time steps. Given $v_{j, k} = l_k y^{(k)}_j$ for each job $j \in F_{k}$ where $|F_{k}| = (1 - \beta) n_k$, exhausting more than $\beta n_k$ jobs ensures we are effectively mimicking PF without being constrained by rigid checkpoints. This provides the intuition for using PMLF.

    \item \textbf{Blind Algorithm}: This policy is non-migratory, ensuring that no job is ever processed on multiple machines throughout its duration. The algorithm is equivalent to that of \cite{AnandGK12} assuming that $p_j = \hat p_j$ for all jobs $j$. 
    The algorithm operates under the assumption of $100\%$ trust in the provided predictions. Let $t$ denote the current time and $q_j(t)$ represent the volume of job $j$ processed prior to $t$. The predicted remaining size of job $j$ is thus given by $\hat{p}_j - q_j(t)$. 

On each machine $i$, we execute the Shortest Job First (SJF) policy (or equivalently SRPT). We define the \textit{residual estimate} for machine $i$ as the total completion time  that would be incurred by running SJF on that machine based on these predicted remaining sizes.

Upon the arrival of a new job with a predicted size, we compute the potential increase in the residual estimate for every machine. The job is then assigned to the machine  that yields the minimum marginal increase to the total residual estimate.

    \item \textbf{Doubling} (Coupled With Immediate-Dispatch): Initially, each job's size is 
assumed to be equal to its predicted size $\hat{p}_j$. If a job's actual 
processing time exceeds this estimate, we double the predicted size 
(i.e., $\hat{p}_j \leftarrow 2 \hat{p}_j$). In this event, we treat 
the original job as completed and consider a new copy of the job to 
have arrived. This copy is dispatched to the machine that results in 
the minimum increase to the residual estimate.

\item \textbf{Hybrid SNAP}: Initially, all jobs are initially in  Group 1. Jobs in Group 1 are dispatched to machines using the Doubling strategy. To loosely emulate the Blind strategy, we utilize large milestones: $\{ c(1+\delta)^i\hat{p}_j \mid i \geq 1\}$ for a large parameter $c \geq 1$. Any job that reaches its first milestone permanently migrates to Group 2. Group 2 jobs are assigned to machines using the SNAP algorithm. Finally, we execute PMLF on individual machines to process their assigned jobs.

\end{enumerate}

\subsection{Further experimental results}

\begin{figure*}[t] %
    \centering
    \begin{minipage}{0.49\textwidth}
        \centering
       \includegraphics[width=\textwidth]{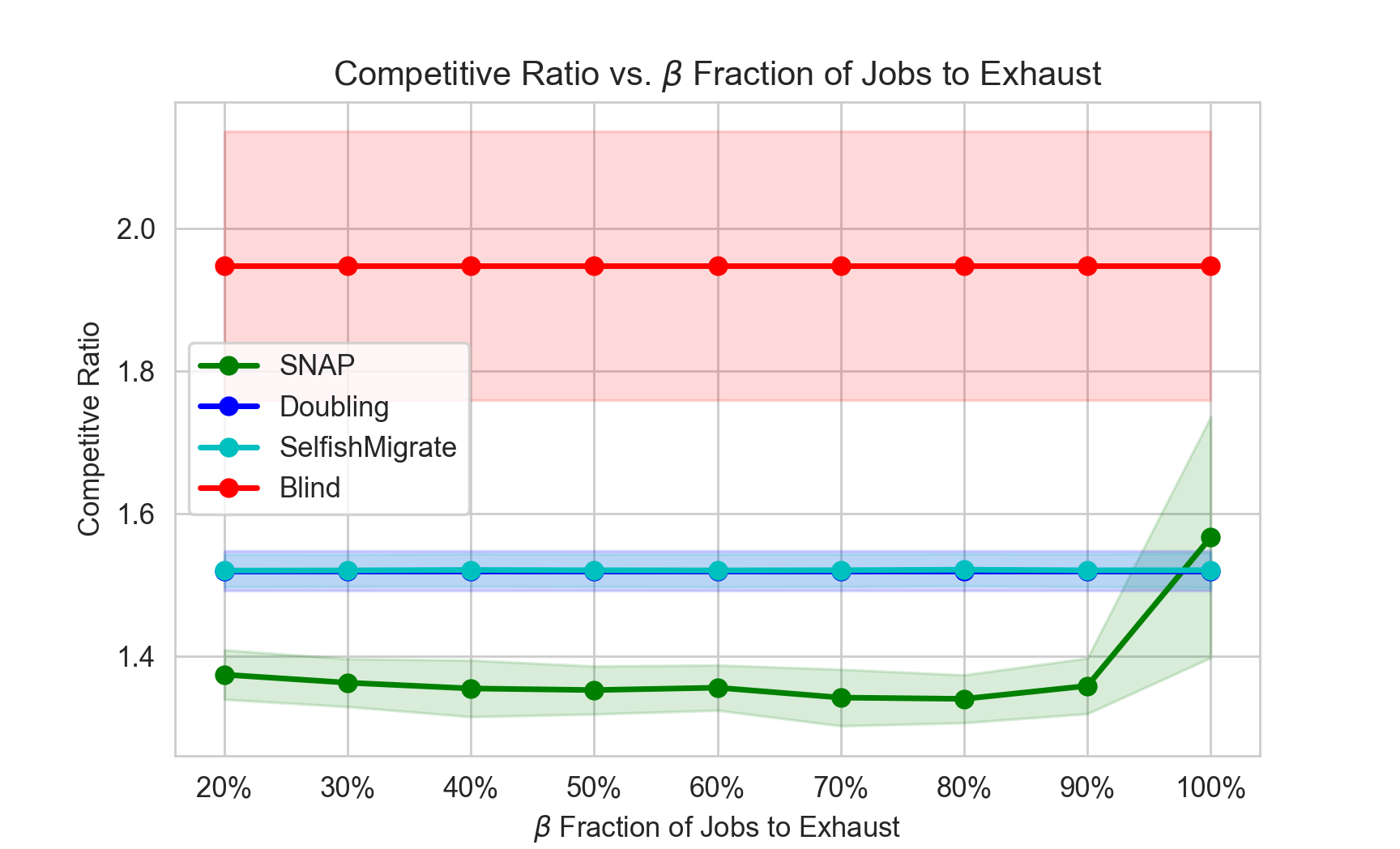}
        \caption{Competitive ratio of SNAP vs. exhaustion parameter $\beta$. Performance remains stable until reaching a $100\%$ exhaustion rate.}
        \label{fig:comp_ratio}
    \end{minipage}
    \hfill
    \begin{minipage}{0.49\textwidth}
        \centering
        \includegraphics[width=\textwidth]{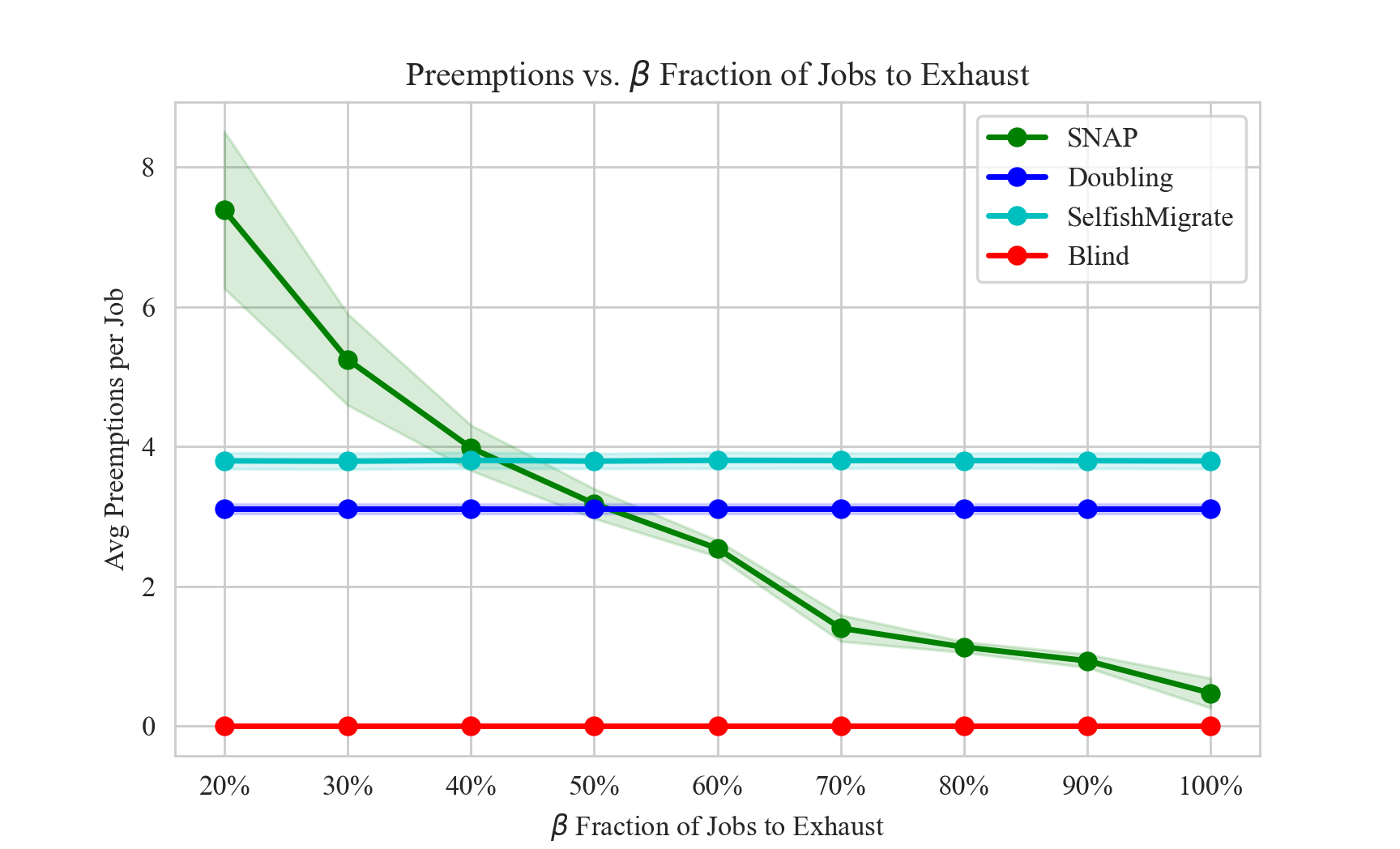}
        \caption{Preemption frequency of SNAP for different values of the exhaustion parameter $\beta$. }
        \label{fig:preemptions}
    \end{minipage}
    \hfill    
\end{figure*}

\begin{figure*}[t] %
    \centering
    \begin{minipage}{0.49\textwidth}
        \centering
       \includegraphics[width=\textwidth]{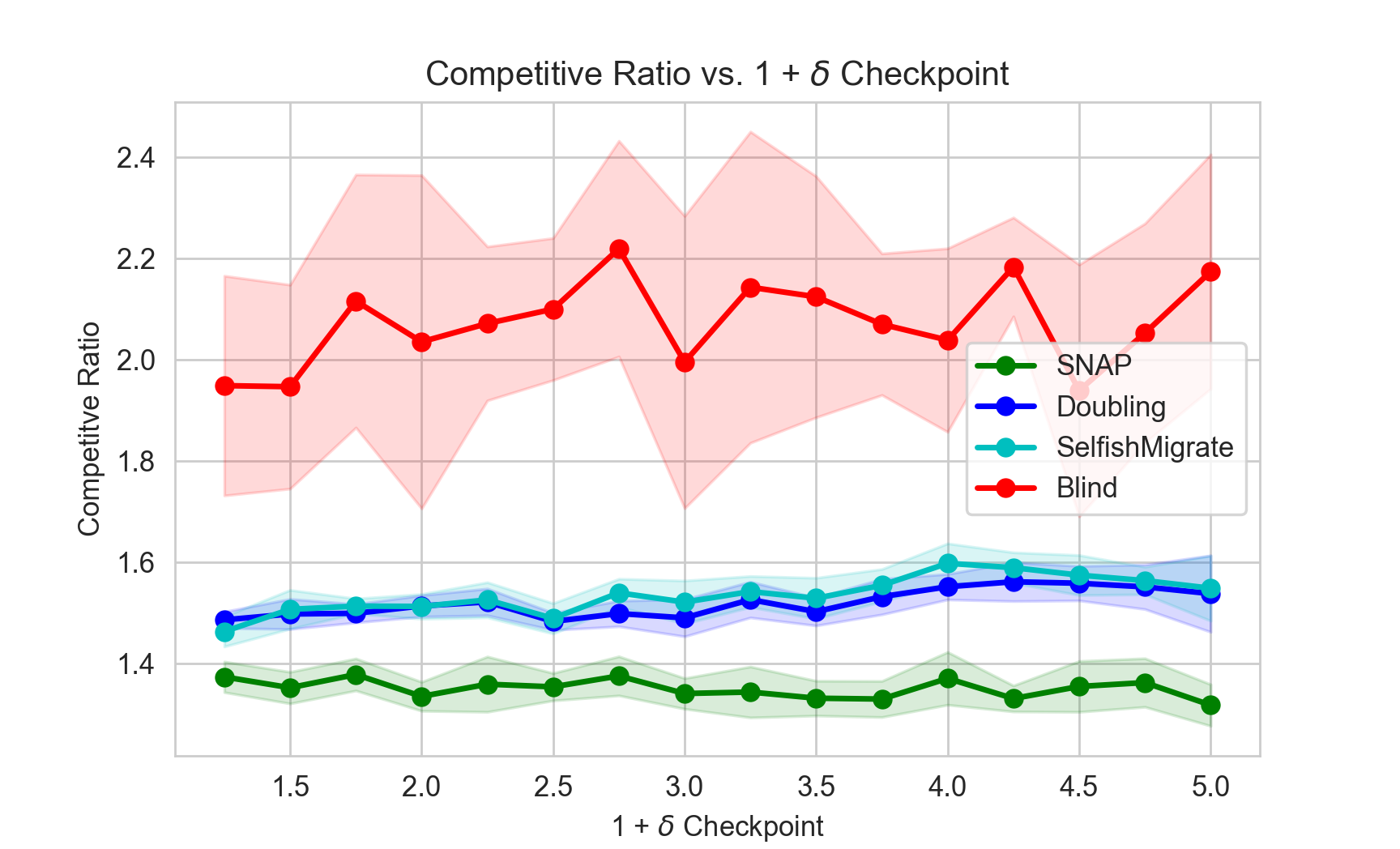}
        \caption{Competitive ratio of the evaluated algorithms for different values of $\delta$.}
        \label{fig:comp_ratio}
    \end{minipage}
    \hfill
    \begin{minipage}{0.49\textwidth}
        \centering
        \includegraphics[width=\textwidth]{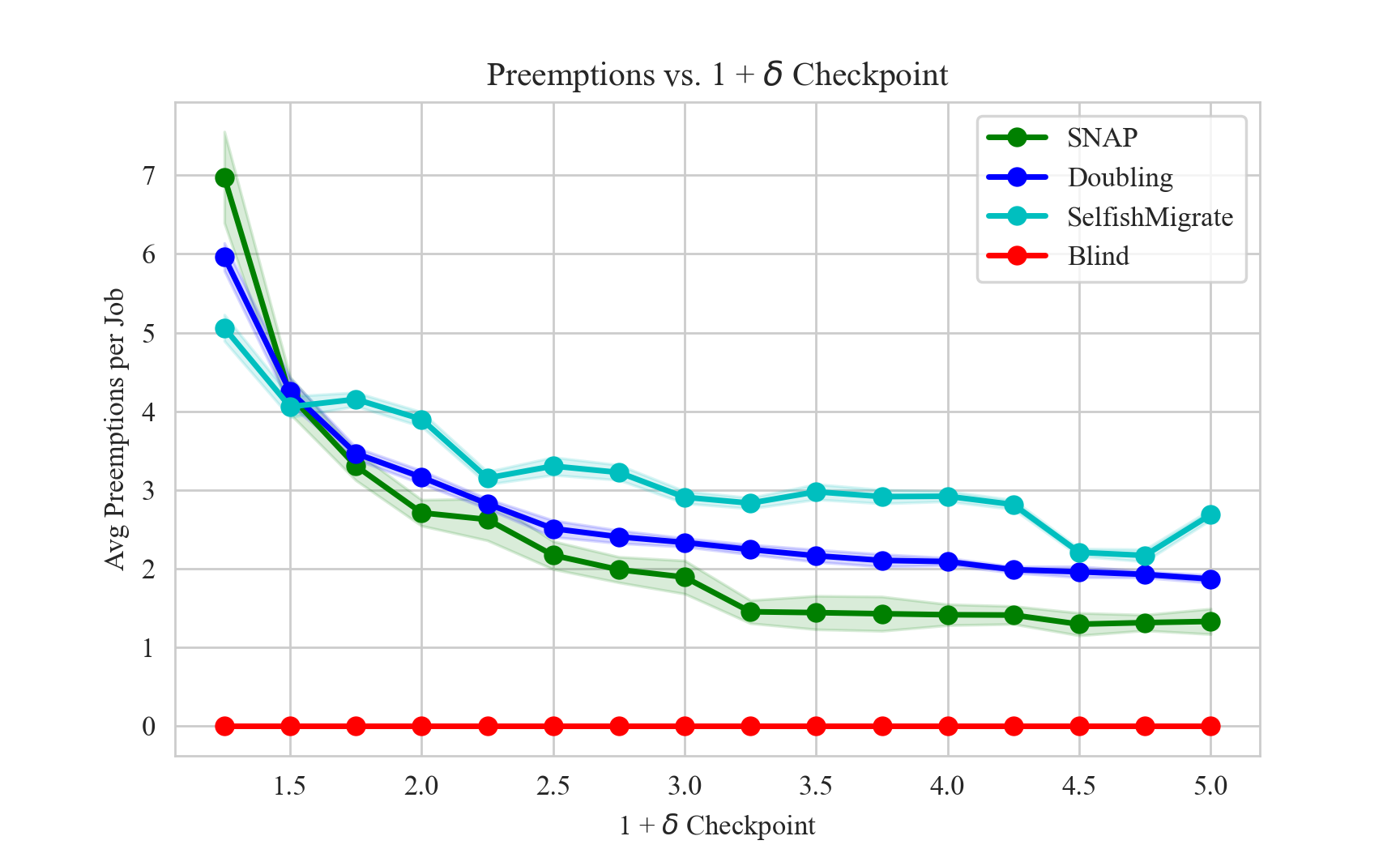}
        \caption{The preemption frequency of the evaluated algorithms for different values of $\delta$.}
        \label{fig:preemptions}
    \end{minipage}
    \hfill    
\end{figure*}

\begin{table}[tbp]
\centering
\caption{Competitive ratio across special job and machine percentages for different algorithms ($\delta = 1$, $\beta = 0.7$, $R = 256$). The Hybrid SNAP (c) algorithm is categorized by a parameter c for larger initial milestones.}  
\small
\setlength{\tabcolsep}{6pt}
\renewcommand{\arraystretch}{1.1}
\begin{tabular}{l c c c c c c}
\toprule
\textbf{Algorithm} & \textbf{0\%} & \textbf{10\%} & \textbf{20\%} & \textbf{30\%} & \textbf{40\%} & \textbf{50\%} \\
\midrule
\hline
Blind & \textbf{1.3076} & 1.7923 & 1.8279 & 1.6468 & 1.6101 & 1.575 \\
Doubling & 1.6983 & 1.5098 & 1.5047 & 1.5375 & 1.5438 & 1.606 \\
\midrule
Hybrid SNAP (1) & 1.5166 & 1.4216 & \textbf{1.3274} & 1.389 & 1.383 & 1.6093 \\
Hybrid SNAP (2) & 1.5353 & 1.6287 & 1.4522 & 1.3347 & 1.4919 & \textbf{1.3347} \\
Hybrid SNAP (4) & 1.4559 & 1.4178 & 1.4824 & 1.4894 & 1.3912 & 1.3448 \\
Hybrid SNAP (6) & 1.5221 & 1.3954 & 1.4255 & \textbf{1.2963} & 1.3854 & 1.3554 \\
Hybrid SNAP (8) & 1.3879 & 1.3994 & 1.3813 & 1.3924 & \textbf{1.3005} & 1.3927 \\
SNAP & 1.4943 & \textbf{1.3562} & 1.359 & 1.4599 & 1.5124 & 1.5315 \\
\bottomrule
\end{tabular}
\end{table}

We conduct further experiments to evaluate the performance of our algorithms. The first experiment investigates the competitive ratio and preemption frequency of SNAP as a function of the exhaustion parameter $\beta$, with $\delta = 1$ and $R = 512$. Interestingly, SNAP exhibits a remarkably stable competitive ratio for most values of $\beta$; the performance degrades sharply only when $\beta$ reaches $100\%$. In contrast, the preemption frequency is highly sensitive to $\beta$, decreasing significantly as $\beta$ increases. Note that since other algorithms do not utilize the $\beta$ parameter, their performance remains constant across these values.

The second experiment compares the algorithms across various values of $\delta$. In our implementation, $\delta$ defines the checkpoints for SNAP and serves as a parameter for the PMLF subprocedure. Intuitively, a larger $\delta$ implies a higher threshold for processing before a preemption occurs. Similarly, our implementation of SelfishMigrate utilizes $\delta$ within its SNAP subprocedure. In the Doubling approach, we simulate a job completion and the release of a new job whenever the processing reaches a power of $(1+\delta)$. Thus, while the algorithms (with the exception of Blind) utilize $\delta$ differently, the parameter serves a similar purpose across all of them. For this experiment, we set $\beta = 0.6$ and $R = 512$. As $\delta$ increases, the competitive ratios of Doubling and SelfishMigrate increase mildly, while that of SNAP decreases slightly. However, the impact of $\delta$ on preemption frequency is more pronounced: the frequency drops consistently as $\delta$ increases across all applicable algorithms.

Our final experiment evaluates Hybrid SNAP by comparing the competitive ratios of the algorithms across different special job and machine frequencies, with parameters set at $\beta = 0.7$ and $R = 256$. We observe that in the absence of special jobs, the Blind algorithm performs effectively. This is because the machines are essentially identical in this scenario, allowing even a blind strategy to achieve efficient load balancing across resources. However, once heterogeneity is introduced (across all tested frequencies from 10\% to 50\%), SNAP and Hybrid SNAP consistently yield the superior competitive ratios. In particular, Hybrid SNAP with $c = 8$ significantly outperforms both Blind and Doubling. This experiment demonstrates the robust performance of our proposed algorithms across varying degrees of system heterogeneity.

\section{Reproducibility and code implementation}
Initial implementations of the experimental code and plotting scripts were generated using Gemini 3 Pro. All generated code was subsequently refactored, manually verified, and rigorously tested to ensure the technical accuracy of the algorithm's implementation and the reproducibility of the results. The full source code is provided in the supplemental material.

\end{document}